\documentclass[titlepage]{article}
\usepackage[margin=1.25in]{geometry}

\usepackage{authblk}

\usepackage{bbm}

\usepackage[font=small,labelfont=bf]{caption}

\renewcommand{\thefootnote}{}
\title{A Gentle Introduction to Gradient-Based Optimization and Variational Inequalities for Machine Learning}
\makeatletter
\renewcommand\AB@authnote[1]{\rlap{\textsuperscript{\normalfont#1}}}

\makeatother
\author[1*]{Neha S. Wadia\thanks{* neha.wadia@berkeley.edu}}
\author[2]{ Yatin Dandi}
\author[3]{Michael I. Jordan}
\affil[1]{\small Center for Computational Mathematics, Flatiron Institute}
\affil[2]{SPOC Laboratory and IdePHICS Laboratory, Ecole Polytechnique F\'{e}d\'{e}rale de Lausanne (EPFL)}
\affil[3]{Department of EECS and Department of Statistics, University of California, Berkeley}

\usepackage{amsthm}
\usepackage{amsfonts}

\newtheorem{theorem}{Theorem}[section]
\theoremstyle{definition}
\newtheorem{definition}{Definition}[section]
\newtheorem{lemma}[theorem]{Lemma}

\theoremstyle{definition}
\usepackage{algorithm}
\usepackage{algpseudocode}
\usepackage{hyperref}

\usepackage{amsmath}
\DeclareMathOperator{\proj}{proj}

\usepackage{tikz}
\usetikzlibrary{intersections}
\usetikzlibrary{arrows,decorations.pathmorphing,backgrounds,positioning,fit,petri}
\usetikzlibrary{arrows,positioning,angles,arrows.meta,calc,intersections}

\usepackage{xcolor}

\hypersetup{
    colorlinks,
    linkcolor={red!50!black},
    citecolor={blue!50!black},
    urlcolor={blue!80!black}
}

\definecolor{purple}{rgb}{0.5,0,0.5}
\definecolor{mulberry}{rgb}{0.77,0.29,0.55}
\definecolor{darkgreen}{rgb}{0,0.5,0}

\usepackage{pgfplots}

\begin{document}

\maketitle
\renewcommand{\thefootnote}{\arabic{footnote}}

\begin{abstract}
The rapid progress in machine learning in recent years has been based on a highly productive
connection to gradient-based optimization.  Further progress hinges in part on a shift in focus from pattern recognition to decision-making and multi-agent
problems.  In these broader settings, new mathematical challenges emerge that involve equilibria
and game theory instead of optima.  Gradient-based methods remain essential---given the high
dimensionality and large scale of machine-learning problems---but simple gradient descent
is no longer the point of departure for algorithm design.  We provide a gentle introduction to a broader framework
for gradient-based algorithms in machine learning, beginning with saddle points and monotone
games, and proceeding to general variational inequalities.  While we provide convergence proofs
for several of the algorithms that we present, our main focus is that of providing motivation
and intuition.
\end{abstract}

\tableofcontents

\pagebreak

\section*{Acknowledgements}
These notes are based on three lectures delivered by Michael Jordan at the summer school ``Statistical Physics and Machine Learning" held in Les Houches, France, in July 2022. Neha Wadia and Yatin Dandi were students at the school. The authors are grateful to Florent Krzakala and Lenka Zdeborov\'{a} for organizing the school.

The authors thank Sai Praneeth Karimireddy for useful discussions clarifying the material on monotone operators, and Sidak Pal Singh for collaboration in the early stages of this manuscript.

MJ was supported in part by the Mathematical Data Science program of the Office of Naval Research under grant number N00014-18-1-2764 and by the Vannevar Bush Faculty Fellowship program under grant number N00014-21-1-2941. NW's attendance at the summer school was supported by a UC Berkeley Graduate Division Conference Travel Grant. The Flatiron Institute is a division of the Simons Foundation.

Author contributions: MJ prepared and delivered the lectures and edited the text. YD partially drafted an early version of the first lecture and made the figures. NW wrote the text and edited the figures.

\pagebreak

\listofalgorithms

\pagebreak

\section{Introduction}

Much research in current machine learning is devoted to ``solving'' or even surpassing human intelligence, with a certain degree of opacity about what that means. A broader perspective shifts the focus from the intelligence of a single individual to instead consider the intelligence of the collective. The goals at this level of analysis may take into consideration overall flow, stability, robustness, and adaptivity. Desired behavior may be formulated in terms of tradeoffs and equilibria rather than optima. Social welfare can be an explicit objective, and we might ask that each individual be better off when participating in the system than they were before.  Thus, the goal would be to \emph{enhance} human intelligence instead of supplanting or superseding it.

The mathematical concepts needed to support such a perspective incorporate the pattern recognition and optimization tools familiar in current machine learning research, but they go further, bringing in dynamical systems methods that aim to find equilibria and bringing in ideas from microeconomics such as incentive theory, social welfare, and mechanism design. A first point of contact between such concepts and machine learning recognizes that learning systems are not only pattern recognition systems, but they are also decision-making systems. While a classical pattern recognition system may construe decision-making as the thresholding of a prediction, a more realistic perspective views a single prediction as a constituent of an active process by which a decision-maker gathers evidence, designs further data-gathering efforts, considers counterfactuals, investigates the relevance of older data to the current decision, and manages uncertainty.  A still broader perspective takes into account the interactions with other decision-makers, particularly in domains in which there is scarcity.  These interactions may include signaling mechanisms and other coordination devices.  As we broaden our scope in this way, we arrive at microeconomic concerns, while retaining the focus on data-driven, online algorithms that is the province of machine learning. 

The remainder of our presentation is organized as follows. We first expand on the limitations of a pure pattern-recognition perspective on machine learning, bringing in aspects of decision-making and social context. We then turn to an exemplary microeconomic model---the multi-way matching market---and begin to consider blends of machine learning and microeconomics. Finally, we turn to a didactic presentation of some of the mathematical concepts that bridge between the problem of finding optima---specifically in the machine-learning setting of nonconvex functions and high-dimensional spaces---and that of finding equilibria.

\subsection{The Challenges of Decision-Making Processes}
\label{sec:lec1-challenges-of-decision-making}

Here is a partial list of obvious challenges that must be addressed when the predictions 
of pattern recognition systems are used in the context of decision-making.
All of these are active areas of research in current machine learning.
\begin{enumerate}
    \item Uncertainty quantification. 
    It is important to be able to give calibrated notions of confidence in the predictions of a model in order to effectively use those predictions to make decisions.
    However, it is not always clear how to systematically assign uncertainty to model predictions, partially because our understanding of the properties of modern model architectures is incomplete.
    Moreover, the traditional and well-studied methods of classical statistics such as the jackknife and the bootstrap are not computationally efficient at the scale of modern applications.
    In Section \ref{sec:lec1-conformal-prediction}, we will discuss some new ideas that help address this problem.
    \item Robustness and adversaries.
    It is important to build models that can provably maintain performance under distribution shift, i.e., differences between test and training data, including differences that are imposed adversarially.
    \item Biases and fairness. We would like to control for bias in our models and ensure that people are treated fairly.
    \item Explainability and interpretability. In a decision-making context, it is often important for the human beings interacting with a model to be able to interpret how the model arrives at its predictions so that they can reason about the relevance of those predictions to their lives.
\end{enumerate}
All the problems we have listed above involve understanding how to design models whose outputs we can guarantee are interpretable, fair, safe, and reliable \emph{for people to interact with}.
The discipline that has historically dealt with this type of problem is not machine learning but economics.

To further illustrate the point, we discuss two specific examples of scenarios where 
the output of a pattern recognition system
is clearly and naturally part of a decision-making process. 

Suppose your physician has access to a trained neural network that takes as input a data vector containing all your health information and generates an output vector of the predicted risk of your developing various diseases. At your next annual medical check-in, you take a standard battery of tests and your physician updates your data vector with the results and then queries the neural network. Let's focus on just one of the predictions of the network: your risk, represented by a number between 0 and 1, of developing a fatal heart condition within the next five years, the only treatment for which is preemptive heart surgery. Let's say that on average people who score a 0.7 or higher develop the condition. The neural network outputs a 0.71 for you. Is there any sense in which this prediction and the accompanying threshold constitutes a decision? Do you immediately get the surgery?

The most likely scenario is that you sit down with your physician to \emph{reason} through how to \emph{interpret} the result and make the most appropriate choice for your healthcare. A number of questions might reasonably come up. What was the provenance of the training data for the neural network? Were there enough contributions to that data from individuals with similar health profiles to yourself to engender some degree of relevance in the prediction? What is the uncertainty in the prediction? What are the sources of that uncertainty? Over the course of the conversation, you may suddenly remember a piece of information that did not seem relevant before---that you had an uncle who died of a suspected heart condition, let's say---but that is now highly relevant. What is the effect of that new data on your reasoning? Furthermore, you may consider counterfactuals. What if you were to change your diet and swim for thirty minutes every day for six months---would that change the prediction, and would it be worth risking postponing any treatment for that period of time in order to ward off the risk of potentially unnecessary heart surgery?

The neural network can clearly help orient this decision-making process, but we cannot naively threshold the output of the model to make a decision, especially because the stakes in this context are high. We also see that the notion of a tradeoff appears naturally in decision-making contexts.

Let us consider another example. Recommendation systems are some of the most successful pattern recognition systems built to date. Consider recommendation systems for books and movies. Is it problematic if a single book or movie is recommended to a large number of people? Not particularly. Books and movies are not scarce. Today, it is possible to print books on demand and to stream movies on demand. Now let's consider a recommendation system for routes to the airport. Is it problematic to recommend the same route to the airport to a large number of people? In this case, the answer is a clear yes. If everyone is directed down the same route, this will cause congestion. The same applies in the case of a recommendation system for restaurants. Recommending the same restaurant to a large number of people will result in lines around the block at a handful of places and an artificial lack of customers at others. A simplistic approach to fixing this type of problem is to treat recommendation in the context of scarcity as a load-balancing problem. The trouble with this approach is that in many cases, it will not be clear on what basis to assign one subset of people to, say, one restaurant versus another. In fact, any such assignment is likely to be arbitrarily made and will introduce further problems with regards to fairness and also economic efficiency.
Once again, we see that the naive approach to making recommendations fails in certain contexts, in particular, in contexts where a scarce resource is being allocated.

We hope we have convinced the reader that no matter how accurate the predictions of a model may be in some technical sense, naive approaches to using those predictions to make decisions do not work well. We argue that a better way forward is to think of the model and the parties that interact with it as constituting a multi-way \emph{market}.

\subsection{Multi-Way Markets}

Let us continue discussing the example of a recommendation system for restaurants. 
A traditional recommendation system uses all the information at its 
disposal, such as the user's browsing history, location history,
a questionnaire about their culinary preferences that the user fills out when they
sign up to use the system, and average customer ratings of all the restaurants in its
database, to recommend a ranked list of restaurants to the user.
We have already discussed that this approach naturally results in competition 
to obtain a table at a small number of restaurants. The problem gets worse as
the number of people using the system increases.

A more efficient, scalable, and flexible way to structure this problem is as
a two-way market, with the restaurants on one side and the customers on the other.
When the customer wishes to find a restaurant for dinner, they \emph{consent}
to share certain information, such as their current location and a few preferences---cuisine, ambiance, and price range, for example---with the restaurants \emph{temporarily}. 
The restaurants then process the incoming information from multiple customers and \emph{make offers} to them,
and perhaps some of them offer a ten percent discount on the total price of dinner
if the customer arrives in the next fifteen minutes. 
Customers then choose the offer that appeals to them. Recommendations emerge organically from an economic system of information sharing and bidding.

Notice the mathematical structure of this problem, which is to compute an optimal
matching between customers and restaurants while accounting for the \emph{interactions}
between and within both groups, and recognizing the \emph{strategic nature} of the participants.  The overall goals of the system design can be those of classical learning systems, including accuracy, coverage, and precision, but they can also include an economic aspect, such as social welfare or revenue.

Two-way matching markets that are not learning-based have had a significant impact in many social and commercial domains.  
A classic example is the matching algorithm that assigns medical residents to hospitals in the United States every year.
It takes as input a ranked list of hospitals from each 
prospective resident, and a ranked list of prospective residents from each
hospital, and computes an optimal matching. Each side of the market must state
their preferences up front. However, it is now recognized that this may not always be possible. One or both
parties may not know their preferences beforehand. An obvious solution suggests
itself, which is for one or both parties to \emph{learn their preferences from data}.
A new subfield of study, called learning-based mechanism design, is now emerging
to address how to combine learning with matching to design new algorithms that can
do both.

The market we described for restaurant matchings was hypothetical.
However, such markets are beginning to become a reality.
An example is the company United Masters, which provides a platform
to match musicians, especially young musicians who may not be represented
by a record label, with their listeners. The platform is structured
as a \emph{three-way} market, involving the musicians, listeners, and 
brands. The latter can use the platform to locate and make contracts 
directly with specific musicians in order to use their music for 
branding purposes. 
The digital marketplace also provides artists with data that informs
them of where the highest numbers of listeners of their music are 
located; they can use this information, for example, to convince venue
owners in the relevant locations to sign them for paid performances. Overall this is a matching system in which data analysis and fairness issues are key.  

Creating markets that work as intended is not a simple task. 
We have already hinted at some of the mathematical challenges involved
in our discussion of restaurant recommendations and United Masters;
problems such as identifying and aligning incentives, designing matching
algorithms in the face of unknown or partially known preferences, and
coming up with alternative business models to advertising-based revenue
generation. In fact, there is a great deal of 
\emph{new} mathematics that will need to be invented in order to enable the
effective design of markets that are beneficial for all parties.

\subsection{Challenges at the Intersection of Machine Learning and Economics}

We provide a few rough and ready examples of mathematical challenges that arise when we
try to develop new machine learning algorithms that bring in ideas from game theory,
mechanism design, and incentive theory.
The goal of addressing these challenges is to discover new principles 
to build healthy learning-based markets that are 
stable over long periods of time.

\begin{enumerate}
\item Understanding relationships among optima, equilibria, and dynamics.
Training a pattern recognition model typically involves finding the 
minimum of a loss function. The dynamics of marketplaces, which can be
described using the language of game theory, typically converge to
\emph{equilibria}. Gradient descent, which provably converges to minima, provably fails to compute equilibria; other algorithms, such as the 
extragradient algorithm, are needed for the latter task. 
We will study both these algorithms in the following lectures.

\item Designing multi-way markets in which individual agents must \emph{explore} to learn their preferences. This involves integrating \emph{bandit algorithms} with matching algorithms.

\item Designing large-scale multi-way markets in which agents view the other sides of the market through interaction with recommendation systems. We described an example of such a marketplace in the context of a recommendation system for restaurants.

\item Uncertainty quantification for blackbox and adversarial settings.
Correctly quantifying uncertainty is crucial to construct safe algorithms
and facilitate effective decision-making.

\item Mechanism design with learned preferences.

\item The design of contracts that incentivize agents to provide data in a federated learning setting.

\item Incentive-aware classification and evaluation.

\item Tradeoffs involving fairness, privacy, and statistical efficiency.
\end{enumerate}

In these lectures, we will focus mostly on the first item in this list.
To give the reader a sense of the flavor of research involved in 
addressing some of the other points in the list, however,
we give brief vignettes in the next section of research in strategic classification and 
distribution-free uncertainty quantification.

\subsection{Two Illustrative Examples}
\label{sec:lec1-research-examples}

\subsubsection{Strategic Classification}
\label{sec:lec1-strategic-classification}

Predictive models are routinely deployed in society for various purposes:
for example, an insurance company may use a predictive model to decide whether or not to insure a prospective client and what premium to charge. Similarly, a college admissions committee evaluates applicants based on some set of criteria to decide whether or not to admit them to the college.
When a predictive model is deployed in society, the people who interact with it will naturally strategize to maximize their utility under the model. Such strategic behavior can eventually render the predictions of the model meaningless. In fact, this effect is so well-known that it bears a name---\emph{Goodhart's Law}---in economics. A classic illustration of the effect is given in Fig.\,\ref{fig:poverty} in the context of the poverty index score \cite{camacho2011manipulation}, a measure of poverty that was designed to be thresholded to determine the subpopulation that qualified for social welfare programs. When the threshold was first introduced, the poverty index score was approximately Gaussian-distributed across the population. However, people soon learned to manipulate their scores in such a way as to increase their chances of placing themselves to the left of the threshold. In a mere decade, the entire peak of the distribution of poverty index scores had moved to the left of the threshold, which now qualified almost the whole population for welfare assistance.

Strategic classification is a problem in which a population that is being classified for some purpose may strategically modify their features in order to elicit a favorable classification. This problem finds a natural formulation as a sequential two-player \emph{Stackelberg game} between a player (the decision-maker) deploying the classifier and a player (the strategic agents) being classified. The goal of the decision-maker is to arrive at the Stackelberg equilibrium, which is the model that minimizes the decision-maker's loss after the strategic agents have adapted to it.

In the prevailing formulation of this game, the decision-maker \emph{leads} by deploying a model, and the strategic agents \emph{follow} by immediately playing their best response to the model. It is assumed that the strategic agents adapt to the model instantaneously. In \cite{zrnic2021leads}, the authors formulate the game more realistically by introducing a finite natural timescale of adaptation for the strategic agents. Relative to this timescale, the decision-maker may choose to play on a timescale that is either faster (a ``reactive" decision-maker) or slower (a ``proactive" decision-maker).\footnote{Note that in this language, most prior work dealt with the setting of a proactive decision-maker.} These timescales are naturally introduced to the problem by formulating the game such that each player must use a learning algorithm to \emph{learn} their strategies, as opposed to playing a strategy from a predetermined set of possibilities. Within this formulation of strategic classification, the authors of \cite{zrnic2021leads} found that the decision-maker can drive the game toward one of two different equilibria depending on whether they chose to play proactively or reactively. Thus the dynamics of play directly influences the outcome of the game. Furthermore, the authors were able to identify certain game settings in which both players prefer the equilibrium that results when the decision-maker plays reactively.

The results of \cite{zrnic2021leads} exemplify how incorporating ideas of statistical learning into a game-theoretic problem can yield a richer model of some social phenomenon that is productive of new insights. We refer the reader to \cite{perdomo-etal} for a broader discussion of this emerging field at the interface of learning and economics.

\begin{figure}
    \centering
    \includegraphics[width=0.6\textwidth]{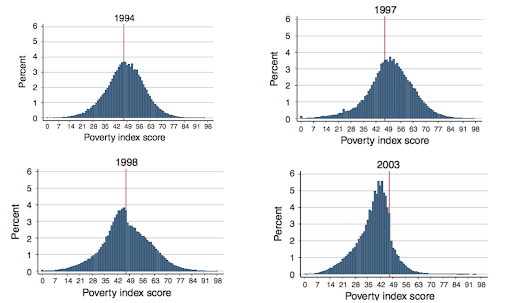}
    \caption{The poverty index score was approximately Gaussian-distributed 
    at the time when it was selected as a measure of poverty that could be thresholded for the purposes of making social policy. With time, its distribution became further and further skewed to the left of the threshold. This figure is reproduced from \cite{camacho2011manipulation}.}
    \label{fig:poverty}
\end{figure}

\subsubsection{Distribution-Free Uncertainty Quantification for Decision-Making}
\label{sec:lec1-conformal-prediction}

As we discussed previously, translating predictions into 
decisions requires reliable uncertainty quantification.
Existing methods in pattern recognition do not provide reliable
error bars, and traditional statistical methods of uncertainty 
quantification such as the bootstrap are prohibitively computationally
intensive at the scales of modern problems.

\emph{Conformal prediction} is a methodology that gives us a way forward.
Given the predictions of a pattern recognition model, such as an image
classifier, a conformal prediction algorithm efficiently computes a 
confidence interval over the outputs of the 
model, thus providing statistical coverage without the need to manipulate 
the model itself or to make strong assumptions about its architecture
or about the distribution of the data.
The interested reader may refer to \cite{angelopoulos2021gentle} for 
an accessible introduction to conformal prediction.

To give the reader a flavor of the mathematics involved in conformal 
prediction, we give a high-level overview of a method of constructing 
\emph{risk-controlling prediction sets} 
\cite{bates2021-risk-controlling-prediction-sets}.
Prediction sets are sets that provably contain the true value of the entity 
being predicted with high probability.
Prediction sets are an actionable form of uncertainty quantification,
and therefore useful for decision-making purposes in certain contexts.
For example, if a physician were to put an image of a patient's colonic 
epithelium through a classifier while trying to diagnose the cause of 
their chronic stomach pain, it would be just as important to know what 
the likely causes of the pain are as to know which causes can safely 
be ruled out.

Consider a dataset $\{(X_i, Y_i)\}_{i=1}^{m}$ of independent, identically 
distributed pairs consisting of feature vectors $X_i\in\mathcal{X}$ and 
labels $Y_i\in\mathcal{Y}$. In many applications, $\mathcal{X}=\mathbb{R}^d$
for some large $d\in\mathbb{N}$. Consider a partition of the dataset into
sets of size $n$ and $m-n$; the former form the \emph{calibration set} $\mathcal{I}_{\text{cal}}$,
the latter the training set $\mathcal{I}_{\text{train}}$.
We are in possession of a trained model $\hat{f}$ (say, a classifier) that was trained on $\mathcal{I}_{\text{train}}$.
The goal is to construct a set-valued predictor $\mathcal{T}_{\lambda}:\mathcal{X}\rightarrow\mathcal{Y}^{\prime}$, where $\mathcal{Y}^{\prime}$ is a space of sets (in many problems, $\mathcal{Y}^{\prime}=2^{\mathcal{Y}}$) and $\lambda\in\mathbb{R}$ is a scalar parameter;
given some user-specified $\gamma, \delta\in\mathbb{R}$, 
for an appropriate choice $\hat{\lambda}$ of $\lambda$, we require the risk $R$ of $\mathcal{T}_{\hat{\lambda}}(X)$ to be bounded above with high probability:
\begin{equation}
    P\left(R\left(\mathcal{T}_{\hat{\lambda}}(X)\right)\leq\gamma\right) \geq 1-\delta.
    \label{eq:lec1-risk-controlled-prediction-set}
\end{equation}
In \cite{bates2021-risk-controlling-prediction-sets}, the authors
provide a procedure by which to build $\mathcal{T}$ and choose $\hat{\lambda}$.

$\mathcal{T}_{\lambda}$ has the following nesting property:
for $\lambda_1 < \lambda_2$, we have 
$\mathcal{T}_{\lambda_1}\subset\mathcal{T}_{\lambda_2}$.
We define a loss function $L:\mathcal{Y}\times\mathcal{Y}^{\prime}\rightarrow\mathbb{R}_{\geq 0}$ on $\mathcal{T}$. 
A typical choice of $L(Y, \mathcal{T}_{\lambda}(X))$ is 
$\mathbbm{1}_{{Y\in\mathcal{T}_{\lambda}(X)}}$.
We assume that the loss function satisfies a monotonicity condition: $\mathcal{T}_{\lambda_1}\subset\mathcal{T}_{\lambda_2}$ implies $L(\mathcal{T}_{\lambda_1})>L(\mathcal{T}_{\lambda_2})$; in other
words, enlarging the prediction set drives down its loss.
The $\emph{risk}$ of the prediction set is the average value of 
the loss over the data: $R\left(\mathcal{T}_{\lambda}(X)\right)=\mathbb{E}L(Y,\mathcal{T}_{\lambda}(X))$.

Using the calibration set $\mathcal{I}_{\text{cal}}$, it is 
possible to construct a quantity $R^{\text{max}}(\lambda)$ such that for all $\lambda$,
$P\left(R\left(\mathcal{T}_{\lambda}(X)\right)\leq R^{\text{max}}(\lambda)\right)\geq 1-\delta$. 
Generally speaking, this construction involves inverting certain
concentration inequalities.
The reader is referred to Section 3 of \cite{bates2021-risk-controlling-prediction-sets} for details.
Once we have $R^{\text{max}}(\lambda)$, taking
\begin{equation}
    \hat{\lambda}=\inf \{\lambda: R^{\text{max}}(\mathcal{T}_{\lambda^{\prime}})<\gamma \,\,\forall\lambda^{\prime}\geq\lambda\}
\end{equation}
ensures that the property \eqref{eq:lec1-risk-controlled-prediction-set} is satisfied. This is the
content of the main theorem of \cite{bates2021-risk-controlling-prediction-sets}.
This formalism is applied to a number of examples in \cite{bates2021-risk-controlling-prediction-sets},
including protein folding, in which case $\hat{f}$ is AlphaFold,
scene segmentation, and tumor detection in the colon.
Note that the assumption of a monotonic loss can be removed, and $\lambda$
can be a low-dimensional vector instead of a number; these generalizations
can be found in \cite{angelopoulos2022-learn-then-test}.

\subsection{Overview of the Lectures}

We hope we have convinced the reader that there are many open mathematical 
challenges involved in combining ideas from machine learning and economics to design
the algorithms we will live among in the future, and that these challenges are worth
tackling. We will devote the rest of these lectures to one of these mathematical 
challenges, namely, the challenge of computing equilibria (in games and other dynamical
systems).

We will present the problem of computing equilibria as a generalization of the problem of 
computing optima.
We will provide a self-contained introduction to both these problems, starting with optima and moving to equilibria.
Our goal is to equip the reader with enough terminology and technical knowledge to read and 
contextualize the primary literature.

Although the algorithms we will discuss are usually implemented in discrete time,
on occasion we will find it useful to study them in continuous time.
As we saw in our discussion of Stackelberg equilibria in Section \ref{sec:lec1-strategic-classification},
the dynamics of an algorithm can influence what the algorithm converges to. This is true
both in optimization and in games, and we will find in each case that the continuous-time
perspective makes this dependency clear. It is also the case that proofs of convergence
are often much simpler to write in continuous time than in discrete time.

We begin with the essentials of convex optimization through the discussion of
two fundamental algorithms: the subgradient method, in Section \ref{sec:lec1-subgrad-method}, 
and gradient descent, in Section \ref{sec:lec1-GD-convex-functions}.
Convex functions are the simplest class of functions for which we can provide
convergence guarantees to the global minimum of the function.
We will see that the subgradient method reduces to gradient descent when the function 
being optimized is differentiable in addition to being convex.
In Section \ref{sec:lec2-GD-nonconvex-functions}, we talk about the convergence properties
of gradient descent on nonconvex functions, which may in general have more than one minimum
and may also have saddle points and maxima.
We will find it useful to have this discussion in continuous time.

Gradient descent is a \emph{first-order} algorithm; its continuous-time limit is a
first-order differential equation in time.
We also provide an exposition on optimization with a description of some recent results on 
\emph{accelerated} algorithms; the continuous-time limits of these algorithms
are second-order differential equations in time, and can be interpreted as the 
equations of motion of a certain Lagrangian.

The dynamics of a marketplace can often be formulated as a game. In order to effectively
design new markets and understand existing ones, it is crucial to be able to compute the
equilibria of the relevant games.
Returning to discrete time, in Section \ref{sec:lec2-finding-nash-equilibria}, we graduate 
from studying optima to studying equilibria. We introduce some basic vocabulary, including
the concept of a variational inequality, and
we show that the Nash equilibria of certain kinds of games can be written as the 
fixed points of variational inequalities involving monotone operators.
In Section \ref{sec:lec3-monotone-operators} we define monotonicity and strong monotonicity,
and in Section \ref{sec:lec3-fixed-point-algorithms} we discuss the proximal point
method and the extragradient algorithm, and their convergence behavior on monotone
operators. We wrap up these lectures with a brief foray back into continuous time
in Section \ref{sec:lec3-highres-ODE}, with a comparative study of a number of 
fixed-point finding algorithms.

Throughout the text, we will omit mathematical details where necessary in the 
interests of clarity and brevity (although we will make note of where we do this). 
The interested reader may find the details of almost everything we discuss here 
in the following references: 
\emph{Convex Analysis} \cite{rockafellar}, 
\emph{Convex Analysis and Monotone Operator Theory in Hilbert Spaces} 
\cite{bauschke2011convex}, 
and \emph{Finite-Dimensional Variational Inequalities and Complementarity Problems} 
\cite{facchinei-pang}.

\subsection{The Subgradient Method and a First Convergence Proof}
\label{sec:lec1-subgrad-method}

We begin our exposition of the mathematics of optimization with
the problem of finding the minimum of a convex function
that need not be differentiable. 
As we have mentioned previously, convexity is the simplest setting in which we can provide convergence guarantees to 
the global minimum of the function. 

Consider the optimization problem
\begin{equation}
    \min_{x\in\mathbb{R}^d} f(x),
    \label{eq:lec1-opt-problem}
\end{equation}
where $f$ is a convex function.
A function is said to be \emph{convex} if the line segment connecting any 
two function values $f(a)$ and $f(b)$ lies above $f(x)$ 
at all points $x$ between $a$ and $b$.
We will give another definition of convexity in the context of differentiable functions 
in the next lecture.
In this section, we do not assume that $f$ is differentiable.
In what follows, we do require some further regularity 
conditions on $f$ beyond convexity (such as convexity of its domain),
but we will not trouble with those details here.

A central notion in convex analysis is that of the \emph{subgradient} of a function.

\begin{definition}[Subgradient]
$g_x$ is a \emph{subgradient} of $f$ at $x$ if,
$\forall y$, $f(y)\geq f(x)+\langle g_x, y-x \rangle$.
This inequality is called the \emph{subgradient inequality}.
\end{definition}

In the definition above and throughout these notes we use $\langle\cdot\rangle$ 
to denote the Euclidean inner product: $\langle x, y\rangle=\sum_{i=1}^dx_iy_i$.

In general the subgradient is not unique. 
For example, the subgradient of the function $f(x) = |x|$ where $x\in\mathbb{R}$ is given by
\begin{equation}
g_x 
\begin{cases}
=1 &x > 0,\\
\in(-1,1) &x = 0,\\
=-1 &x < 0.
\end{cases}
\end{equation}
For this reason, it is also useful to give a name to the set of all
subgradients of a function, which we do next.
\begin{definition}[Subdifferential]
The subdifferential $\partial f(x)$ of a function $f$ at a point $x$ is the set of all subgradients of $f$ at $x$, i.e.,
\begin{equation*}
    \partial f(x)= \{g_x:\forall y, f(y)\geq f(x)+\langle g_x, y-x \rangle\}.
\end{equation*}
\end{definition}

If $f$ is differentiable at $x$, then $\partial f(x)$ contains a single element equal
to the derivative of $f$ at $x$. We remind the reader that in this section
we do not assume that $f$ is differentiable. We are now ready to introduce
the subgradient method for solving \eqref{eq:lec1-opt-problem}.

\begin{algorithm}
\caption{Subgradient Method with Constant Step Size}
\label{alg:lec1-subgrad}
\begin{algorithmic}
% \Ensure $x_1$ s.t. $||x_1-x^{\star}||\leq R$.
\Require $\eta>0, \, \, T>0$\\
\While{$k \leq T$}\\
    $x_{k+1}=x_k-\eta g_k$\\
    $k=k+1$
\EndWhile
\end{algorithmic}
\end{algorithm}
In Algorithm \ref{alg:lec1-subgrad},
$T$ is the number of iterations for which we run the algorithm, 
$\eta$ is the (constant) step size,
$g_k$ is shorthand for $g_{x_k}$, and $g_k\in\partial f(x_k)$.
We will often refer to $x_k$ as ``the $k$-th \emph{iterate}''.

The subgradient method is \emph{not} a descent method: it is not
guaranteed to make progress toward the optimum with each iteration.
This is different from the case of gradient descent, which we will discuss
in the sequel (see Lemma \ref{lem:lec2-descent-GD-smooth}).
Nonetheless, we can prove that it converges in function value with 
a rate $1/\sqrt{T}$.
Let $x^{\star}$ denote the optimum (i.e., the solution to 
\eqref{eq:lec1-opt-problem}), and $||\cdot ||$ the Euclidean norm,
and consider the quantity $||x_{k+1}-x^{\star}||^2$:
\begin{align}
    ||x_{k+1}-x^{\star}||^2 
    &= ||x_k-\eta g_k -x^{\star}||^2 \nonumber\\
    &=||x_k-x^{\star}||^2 - 2\eta \langle g_k, x_k-x^{\star}\rangle + \eta^2 ||g_k||^2 \nonumber\\
    &\leq ||x_k-x^{\star}||^2 -2\eta\left(f(x_k)-f(x^{\star})\right) + \eta^2 G^2.
    \label{eq:lec1-SG-inequality}
\end{align}
To arrive at \eqref{eq:lec1-SG-inequality}, we applied the subgradient inequality 
and assumed that $||g_k||^2$ is bounded above by a constant $G^2$.
We now move $f(x_k)-f(x^{\star})$ to the left-hand side, sum over $k$, 
and divide by $2\eta T$:
\begin{equation}
\frac{1}{T} \sum_{k=1}^T \left(f(x_k)-f(x^{\star})\right)  
\leq \frac{1}{2\eta T}\sum_{k=1}^{T}\left(||x_k-x^{\star}||^2-||x_{k+1}-x^{\star}||^2\right) + \frac{1}{2}\eta G^2.
\label{eq:lec1-subgrad-intermediate-step}
\end{equation}
On the right-hand side of \eqref{eq:lec1-subgrad-intermediate-step}
we have a \emph{telescoping 
sum}, in which all but the first and last terms cancel.
On the left-hand side, we take advantage of the convexity of $f$ and
apply \emph{Jensen's inequality}, which states that $f$ evaluated at the 
average of a set of points $x_i$ is at most the average of the
function values at those points (for $f$ convex). 
We arrive at
\begin{align}
    f\left(\frac{1}{T} \sum_{k=1}^T x_k\right)-f(x^{\star})
    &\leq \frac{1}{2\eta T}\left(||x_1-x^{\star}||^2-||x_{T+1}-x^{\star}||^2\right) + \frac{1}{2}\eta G^2
    \nonumber\\
    &\leq \frac{1}{2\eta T}R^2 + \frac{1}{2}\eta G^2.
    \label{eq:lec1-subgrad-intermediate-step-2}
\end{align}
In the last step above, we dropped the term 
$||x_{T+1}-x^{\star}||^2$ and assumed that $||x_1-x^{\star}||^2$
is bounded above by a constant $R^2$.
To obtain a rate of convergence, we choose $\eta$ to minimize the
right-hand side of the last inequality. The optimal value of $\eta$
turns out to be $\frac{R}{G\sqrt{T}}$. Substituting this in \eqref{eq:lec1-subgrad-intermediate-step-2}, we have
\begin{equation}
    f\left(\frac{1}{T} \sum_{k=1}^T x_k\right)-f(x^{\star})
    \leq \frac{RG}{\sqrt{T}}.
    \label{eq:lec1-subgrad-method-rate}
\end{equation}
We have thus proved the following result.
\begin{theorem}
\label{thm:lec1-subgradient-convex}
On a convex function $f$, the subgradient method converges
in function value on the average iterate with a rate $1/\sqrt{T}$, 
where $T$ is the number of iterations.
\end{theorem}

A remarkable fact about \eqref{eq:lec1-subgrad-method-rate} is 
that it has no explicit dependence on the dimension of $x$ (it depends implicity on the dimension through $R$ and $G$). It can further be shown that 
the $1/\sqrt{T}$ rate is tight; i.e., that there exist convex functions 
that saturate the upper bound in \eqref{eq:lec1-subgrad-method-rate}.
Lastly, we note that in order to extract the $1/\sqrt{T}$ rate, we had to choose a step size that
depended on $T$. We will not see this type of
dependence of the step size on the total number of iterations again
in these lectures.

This discussion of the subgradient method introduces
a pattern that will recur as we continue our study of optimization in the 
subsequent sections: we make certain regularity assumptions on 
the function and write down an algorithm to optimize it; we then exploit the
dynamics of the algorithm as well as the regularity properties of
the function to prove convergence.  Typically, we are interested in
the rate at which the distance between the iterate and the optimum
shrinks, and in particular, we are interested in algorithms that
achieve a \emph{fast rate}.

\pagebreak

%%%%%%%%%%%%%%%%%%%%%%% LECTURE 2 %%%%%%%%%%%%%%%%%%%%%%

\section{Computing Optima in Discrete and Continuous Time}

\subsection{Convergence Guarantees for Gradient Descent on Convex Functions}
\label{sec:lec1-GD-convex-functions}

We remind the reader that we are currently discussing algorithms
for the following optimization problem
\begin{equation*}
    \min_{x\in\mathbb{R}^d} f(x).
\end{equation*}
In Section \ref{sec:lec1-subgrad-method}, we assumed that $f$ is convex
but not necessarily differentiable, and we studied the
convergence behavior of the subgradient method (Algorithm \ref{alg:lec1-subgrad})
on this problem.
When $f$ is a \emph{smooth} function of $x$, the subgradient
method reduces to the well-known gradient descent algorithm.

\begin{algorithm}
\caption{Gradient Descent}
\label{alg:lec2-GD}
\begin{algorithmic}
\Require $\eta>0$\\
\State $x_{k+1}=x_k-\eta \nabla f(x_k)$
\end{algorithmic}
\end{algorithm}

We will now give a definition of convexity for differentiable
functions, and then explore the consequences on the convergence behavior
of gradient descent of two further assumptions 
we can make on $f(x)$: (1) Lipschitz smoothness, and (2) strong convexity.
The proof techniques that arise will be useful when we study variational inequalities.

\begin{definition}[Convexity]
A (differentiable) function $f$ is \emph{convex} on a set if and only if
\begin{equation}
    f(x)\geq f(y) + \langle\nabla f(y), x-y\rangle
    \label{eq:lec2-convexity-def}
\end{equation}
for all $x$, $y$ in the set.
\end{definition}

There are several other equivalent definitions of convexity. Similarly, there are many ways to characterize smoothness. One possibility is the following.

\begin{definition}[Smoothness]
\label{lec2-def-smoothness}
A function $f$ is \emph{(Lipschitz) smooth} if it satisfies
\begin{equation}
||\nabla f(x) - \nabla f(y)|| \leq L ||x-y||,
\label{eq:lec2-smoothness-def}
\end{equation}
where $L>0$. Such a function is often referred to as $L$-smooth.
\end{definition}
Note that \eqref{eq:lec2-smoothness-def} is equivalent to
\begin{equation}
f(y) \leq f(x) + \langle\nabla f (x), y-x\rangle +\frac{L}{2}||x-y||^2.
\label{eq:lec2-smoothness}
\end{equation}
The proof of this equivalence is left as an exercise to the reader. 
\emph{Hint:} use the fundamental theorem of calculus and the Cauchy-Schwarz inequality.
The inequality \eqref{eq:lec2-smoothness} indicates that an $L$-smooth
function can be upper-bounded by a quadratic at every point 
$y$ in its domain.
With this intuition it is simple to see why, for example, 
the function $f(x)=|x|$ is \emph{not} smooth: at the point $x=0$,
there is no quadratic that both touches $f$ only at $x=0$ and 
sits above $f$ for all $x>0$.

We will now see that gradient descent on $L$-smooth convex functions 
converges with a rate $1/T$, where $T$ is
the total number of iterations. This is faster than
the $1/\sqrt{T}$ rate than we had in \eqref{eq:lec1-subgrad-method-rate}, when we did not assume smoothness.

We begin by proving a type of result called a \emph{descent lemma}, which 
guarantees that we
make a certain amount of progress in decreasing the function 
value at each iteration of gradient descent.

\begin{lemma}[Descent lemma for gradient descent on smooth convex functions]
\label{lem:lec2-descent-GD-smooth}
Consider an $L$-smooth convex function $f$. Under the dynamics of Algorithm \ref{alg:lec2-GD} with $\eta = L^{-1}$, we have
\begin{equation}
    f(x_{k+1}) \leq  f(x_k) - \frac{1}{2L}||\nabla f(x_k)||^2.
    \label{eq:lec2-descent-GD-smooth}
\end{equation}
\end{lemma}
\begin{proof}
Take $y=x_{k+1}$ and $x=x_k$ in \eqref{eq:lec2-smoothness}, and replace $x_{k+1}-x_{k}$ with $-L^{-1}\nabla f(x_k)$.
\end{proof}

\begin{theorem}
\label{thm:lec2-GD-smooth-functions}
On an $L$-smooth convex function $f$, 
gradient descent with step size $L^{-1}$ converges
in function value with a rate $1/T$, where $T$ is the number of iterations.
\end{theorem}
\begin{proof}
Beginning with Lemma \ref{lem:lec2-descent-GD-smooth}, we have
\begin{align}
    f(x_{k+1}) 
    &\leq  f(x_k) - \frac{1}{2L}||\nabla f(x_k)||^2\nonumber\\
    &\stackrel{(i)}{\leq} f(x^{\star}) + \langle \nabla f(x_k), \, x_k-x^{\star} \rangle - \frac{1}{2L}||\nabla f(x_k)||^2\nonumber\\
    &\stackrel{(ii)}{=} f(x^{\star}) + \frac{L}{2} \left(
    ||x_k-x^{\star}||^2 - \left|\left|x_k-x^{\star}-\frac{1}{L}\nabla f(x_k)\right|\right|^2
    \right)
    \nonumber\\
    &\stackrel{(iii)}{=} f(x^{\star}) + \frac{L}{2}\left(||x_k-x^{\star}||^2-||x_{k+1}-x^{\star}||^2\right).
    \label{eq:lec2-telescoping-inequality}
\end{align}
In $(i)$, we applied the definition of convexity (take $y=x_k$ and $x=x^{\star}$ in \eqref{eq:lec2-convexity-def}). 
In $(ii)$, we completed the square, and in $(iii)$, we recognized that $x_k-L^{-1}\nabla f(x_k)$ is simply a gradient descent update with step size $L^{-1}$ and replaced it with $x_{k+1}$.
Now we move $f(x^{\star})$ to the left-hand side and sum \eqref{eq:lec2-telescoping-inequality} over $k$.
On the right-hand side, this results in a telescoping sum.
We are left with the following:
\begin{equation}
    \sum_{k=1}^{T}\left(f(x_{k+1})-f(x^{\star})\right)
    \leq \frac{L}{2}\left(
    ||x_1-x^{\star}||^2-||x_{T+1}-x^{\star}||^2\right).
    \label{eq:lec2-telescoping-sum}
\end{equation}
We can drop the second term on the right-hand side of \eqref{eq:lec2-telescoping-sum}, since it is nonpositive.
Note that Lemma \ref{lem:lec2-descent-GD-smooth} implies $f(x_{k+1})-f(x^{\star})\leq f(x_{k})-f(x^{\star})$ for all $k$. 
Each term in the sum on the left-hand side of \eqref{eq:lec2-telescoping-sum} is therefore lower-bounded by the quantity $f(x_T)-f(x^{\star})$.
Using this lower bound, \eqref{eq:lec2-telescoping-sum} implies
\begin{equation}
    T\left(f(x_T)-f(x^{\star})\right)
    \leq \frac{L}{2}||x_1-x^{\star}||^2.
\end{equation}
Lastly, we bound the distance between $x_1$ and the optimum by some $R>0$, and divide through by $T$.
We thus arrive at the result
\begin{equation}
    f(x_T)-f(x^{\star})\leq \frac{RL}{2}\frac{1}{T}.
\end{equation}
\end{proof}
A variant of Theorem \ref{thm:lec2-GD-smooth-functions} can be found in 
\cite{NesterovBook}, Corollary 2.1.2.

In Section \ref{sec:lec1-subgrad-method}, assuming only convexity, we saw that 
the subgradient method converges in average function value with the rate 
$1/\sqrt{T}$. Since the subgradient method 
reduces to gradient descent when $f$ is differentiable,
we have just argued that the same method achieves a faster rate, $1/T$,
when we assume $f$ is $L$-smooth in addition to being convex.
If we were to make a third assumption, that $f$ is not just convex but
\emph{strongly convex},
we would obtain an even faster rate of convergence. 
We present the definition of strong convexity
next, and state the corresponding convergence result for gradient descent.

\begin{definition}[Strong convexity]
A function $f$ is said to be \emph{strongly convex} if there exists $\mu>0$ such that
\begin{equation}
    f(y)\geq f(x) + \langle\nabla f(x), y-x\rangle + \frac{\mu}{2}||x-y||^2.
    \label{eq:lec2-strong-convexity-def}
\end{equation}
Such a function is said to be $\mu$-strongly convex.
\end{definition}

All strongly convex functions are also convex.

\begin{theorem}
\label{thm:GD-on-strongly-convex-f}
On an $L$-smooth and $\mu$-strongly convex function $f$,
gradient descent with step size $\eta\leq 2/(L+\mu)$ converges
at the rate $(1-C)^T$, where $C\leq 1$ is a constant that depends on 
$\eta$, $\mu$, and $L$, and $T$ is the total number of iterations.
In particular,
\begin{equation}
    ||x_{T+1}-x^{\star}||^2\leq\left(1-\frac{2\eta\mu L}{\mu+L}\right)^T||x_1-x^{\star}||^2.
\end{equation}
\end{theorem}
See \cite{NesterovBook}, Theorem 2.1.15, for a proof.
The rate in Theorem \ref{thm:GD-on-strongly-convex-f} is known
as a \emph{linear} rate of convergence in optimization parlance.

\subsection{Gradient Descent on Nonconvex Functions: Escaping Saddle Points Efficiently}
\label{sec:lec2-GD-nonconvex-functions}

So far we have dealt with convergence guarantees for gradient 
descent on convex functions. 
In this section we investigate guarantees for the performance of gradient descent on \emph{nonconvex} functions,
and study how these differ from the guarantees we can make for convex
functions. One of the differences, we will see, is that when $f$ is 
nonconvex, convergence guarantees for gradient descent can have a 
dimension dependence. The discussion in this section is 
based on joint work \cite{jin17a-icml-escape-saddles,jin2021-acm-nonconvex} with Chi Jin, Rong Ge, Praneeth Netrapalli, and Sham Kakade.

Consider an $L$-smooth function $f(x):\mathbb{R}^d\rightarrow\mathbb{R}$.
We assume that $f$ also has $\rho$-Lipschitz Hessians:
\begin{equation}
    ||\nabla^2 f(x)-\nabla^2 f(y)||\leq\rho||x-y|| \,\,\, \forall x,y.
\end{equation}
(We say that $f$ is $\rho$-Hessian Lipschitz.)
Gradient descent halts naturally when $\nabla f(x^{\star})=0$.
When $f$ is nonconvex, this can happen not just when $x^{\star}$ 
is a minimum, but also if it is a maximum or a saddle point.
We need therefore to understand whether gradient descent can efficiently
avoid maxima and saddle points. 
Here, in particular, we focus on how to efficiently
escape saddle points.
Note that in the nonconvex setting, we will no longer be looking for 
convergence guarantees to global minima, but to local minima.
The guarantees thus obtained are useful nonetheless. In many applications, it is the case that $f$ does not have any spurious
local minima; principal components analysis is a notable example.

First, we introduce some terminology:
\begin{itemize}
    \item A \emph{first-order stationary point} of $f$ is any point $x$ such that $\nabla f(x)=0$.
    \item A \emph{second-order stationary point} of $f$ is any point $x$ such that $\nabla f(x)=0$ and $\nabla^2 f(x)\succeq 0$.
    \item An $\varepsilon$-\emph{first-order stationary point} ($\varepsilon$-FOSP) is 
    a point $x$ such that $\nabla f(x)\leq\varepsilon$.
    \item An $\varepsilon$-\emph{second-order stationary point} ($\varepsilon$-SOSP) is an $\varepsilon$-FOSP where $\nabla^2 f(x)\succeq - \sqrt{\rho\varepsilon}I$. Here, $I$ is the $d$-dimensional identity matrix.
\end{itemize}
Note that saddle points are first-order stationary points but not
second-order stationary points.

\begin{definition}[Strict saddle point]
$x$ is a \emph{strict saddle point} of $f$ if it is a saddle point and the 
smallest eigenvalue of the Hessian at $x$ is strictly negative:
$\lambda_{min}(\nabla^2f(x)) <0$.
\end{definition}

In the following, we will discuss convergence guarantees for a variant of
gradient descent that avoids/escapes strict saddle points. 
That is, we will not discuss how to escape saddles that have
no escape directions (i.e., at which $\lambda_{min}(\nabla^2f(x)) =0$).

We will be talking about convergence rates in a different way than we have so far. 
Instead of bounding the error in function value (or distance to the optimum)
after $T$ iterations, we will ask how many iterations
are necessary for the error to drop below some $\varepsilon$. 
Let us restate the convergence guarantee of Theorem \ref{thm:lec2-GD-smooth-functions} in this manner:
\begin{theorem}
\label{thm:lec2-GD-smooth-f-iters}
On an $L$-smooth convex function $f$, gradient descent with step size $\eta=L^{-1}$ converges to an $\varepsilon$-FOSP in $(f(x_1)-f^{\star})L/\varepsilon^2$ iterations.
\end{theorem}

Previous results on optimizing nonconvex functions with gradient
descent had established that the algorithm asymptotically avoids saddle points \cite{lee2016,lee2019}. However, these results 
made no finite-time guarantee of convergence.
Other work then established that gradient descent can escape saddles in a length
of time that is polynomial in the dimension $d$
\cite{ge15}.
In \cite{jin17a-icml-escape-saddles}, we showed that this dimension
dependence can be improved. We developed
a variant of gradient descent we call perturbed gradient descent 
(PGD) that converges with high probability to $\varepsilon$-SOSPs at 
a rate that has a polylogarithmic dependence on $d$.
PGD takes an injection of noise when the norm of the gradient 
of $f$ is small, 
and otherwise behaves like vanilla gradient descent.
The injected noise is sampled uniformly from a $d$-dimensional ball of radius $r$. (For details, see the description of Algorithm 2 in \cite{jin17a-icml-escape-saddles}.)
We prove the following theorem guaranteeing convergence of PGD to $\varepsilon$-SOSPs.
\begin{theorem}[PGD escapes strict saddle points]
\label{thm:lec2-PGD-convergence}
Consider an $L$-smooth and $\rho$-Hessian Lipschitz function $f:\mathbb{R}^d\rightarrow\mathbb{R}$. There exists an absolute 
constant $c_{max}$ such that for any $\delta>0$, $\varepsilon< L^2/\rho$, $\Delta_f\geq f(x_1)-f^{\star}$, and $c\leq c_{max}$, PGD with step size $\eta=L^{-1}$ will output an $\varepsilon$-SOSP with probability $1-\delta$, and will terminate in a number of iterations of order
\begin{equation}
    \frac{(f(x_1)-f^{\star})L}{\varepsilon^2}\log^4\left(\frac{dL\Delta_f}{\varepsilon^2\delta}\right).
    \label{eq:lec2-PGD-iteration-count}
\end{equation}
\end{theorem}
The quantity $c$ is an important hyperparameter of PGD, although it does 
not appear in \eqref{eq:lec2-PGD-iteration-count}.
See Theorem 3 in \cite{jin17a-icml-escape-saddles} for the proof of Theorem \ref{thm:lec2-PGD-convergence}.

This result implies that the presence of strict saddles does not
slow gradient descent down by much: the number of iterations needed to 
converge goes as $1/\varepsilon^2$, which is precisely the same dependence
we find for gradient descent on smooth convex functions (see Theorem \ref{thm:lec2-GD-smooth-f-iters}).

We give a brief description of the proof technique.
In the vicinity of a saddle point, there exists a band (a ``stuck region") in which the
flow of the vector field $\nabla f$ is mostly toward the saddle, i.e., where PGD can get stuck.
The core technical part of the proof 
involved bounding the thickness of this band.
To do so, we studied the coupling time of two Brownian motions 
initialized a distance $k$ apart.
When $k$ is small enough, the diffusions eventually couple. Otherwise, they never couple. Essentially, there is a phase transition in the coupling time that depends on how far apart the diffusions are initialized. The value of this initial distance at the phase transition 
can be bounded, and this bound gives us control over the thickness of the stuck region.

Achieving tighter control than was previously possible over the thickness of the stuck region
is precisely what enables us to improve on the polynomial dimension
dependence of the convergence rate from previous work \cite{ge15}.
We note that the requirement of $\rho$-Hessian Lipschitzness is absolutely
crucial to the analysis. Without this assumption, PGD cannot find the negative
directions it must follow to escape the vicinity of a saddle.

Two open questions follow. The first is, is the polylogarithmic dimension dependence
of the convergence rate simply an artifact of the proof technique? Since we 
have neither a lower bound on the dimension dependence, nor an alternative
proof technique, we do not know the answer to this question. The second is,
is there a way to bound the thickness of the stuck region that relies not
on ideas from probability theory but on ideas from differential geometry?

\subsection{Variational, Hamiltonian, and Symplectic Perspectives on Acceleration}
\label{sec:lec2-cnts-time-acceleration}

On a given class of functions, is there a limit to how fast an
algorithm can reach the optimum? 
This type of question is addressed by proving a \emph{lower bound} \cite{nemirovski-yudin1983} on the convergence rate. Lower bounds are 
usually established by studying worst-case function instances.
For the class of smooth and strongly convex functions, the 
lower bound is a rate of $1/T^2$, which is faster than the $1/T$ rate achieved
by gradient descent, discussed in Theorem \ref{thm:lec2-GD-smooth-functions}. This optimal rate is achieved by the 
family of second-order or so-called \emph{accelerated} algorithms.
A representative member of this family is Nesterov's method (Algorithm \ref{alg:lec2-nesterov}).

\begin{algorithm}
\caption{Accelerated Gradient Descent (Nesterov's Method)}
\label{alg:lec2-nesterov}
\begin{algorithmic}
\Require $\eta >0$\\
\State $x_{k+1} = (1-\lambda_k)y_{k+1}+\lambda_k y_k,$ where\\
\\
$\lambda_k = \frac{k-1}{k+2}, \,\,\, y_{k+1} = x_k - \eta\nabla f(x_k)$
\end{algorithmic}
\end{algorithm}

On a technical level, accelerated algorithms are somewhat mysterious.
Proofs of convergence of these methods are not easy to motivate
or parse. They rely on the technique of estimate sequences, first developed by Nesterov \cite{NesterovBook} and later generalized by Baes \cite{baes-estimate-sequences}.

We will now discuss a series of results in which we take a continuous-time view of acceleration. 
We will see that this viewpoint will allow us to interpret acceleration in discrete time
as arising from the discretization mechanism of a specific second-order differential equation.
In general, we believe that taking a continuous-time viewpoint in
optimization can provide a unifying perspective and enable results that
we do not yet know how to prove in discrete time. We will see examples of 
this later in the discussion.

By taking the limit $\eta\rightarrow 0$, we can derive continuous-time
limits for discrete-time algorithms. When the latter are deterministic, the former are ordinary differential equations. It is simple to show, 
for example, that in continuous time, gradient descent maps to the gradient flow equation,
\begin{equation}
    \dot{x}_t = -\nabla f(x_t).
\end{equation}
We use dots to denote derivatives with respect to time.
Nesterov's method maps
to the following second-order differential equation, first written down in \cite{su-boyd-candes}:
\begin{equation}
    \ddot{x}_t + \frac{3}{t}\dot{x}_t + \nabla f(x_t) = 0.
    \label{eq:lec2-su-boyd-candes-ode}
\end{equation}
Note that $x_t$ is shorthand for $x(t)$, and is not to be confused with 
the discrete-time iterate $x_k$. 
The authors of \cite{su-boyd-candes} also showed that the Euler discretization
of \eqref{eq:lec2-su-boyd-candes-ode} recovers Algorithm \ref{alg:lec2-nesterov}.
This result provided some insight, and also gave rise to new
questions and problems. We will briefly describe one such question
and its resolution before returning to the main theme of this section.

More than one discrete-time algorithm collapses 
to \eqref{eq:lec2-su-boyd-candes-ode} in the continuous-time limit. 
In particular, the heavy ball method, a second-order algorithm due to 
Polyak \cite{polyak1964},
also maps to \eqref{eq:lec2-su-boyd-candes-ode}. The trouble with this is
that the heavy ball method does not achieve the optimal $1/T^2$ rate
of Nesterov's method, indicating that \eqref{eq:lec2-su-boyd-candes-ode}
fails to distinguish between two algorithms that achieve different
rates in discrete time. This problem can be addressed by taking
\emph{high-resolution} continuous-time limits. High-resolution ordinary
differential equations appear in fluid dynamics and are useful
to describe systems with multiple timescales. See \cite{shi-high-res-ODE}
for details on how to use high-resolution limits to derive distinct 
continuous-time analogs of the heavy ball method and Nesterov's method.
We will revisit high-resolution differential equations in another
context at the end of the third lecture, in Section \ref{sec:lec3-highres-ODE}.

Let us return to the primary point of discussion. We have mentioned
that a continuous-time limit of Nesterov's method was studied in
\cite{su-boyd-candes} and that taking the continuous-time point of view
was productive of insight into the behavior of the algorithm.
Our work, which will now describe, was motivated by the question of 
whether there exists a \emph{generative mechanism} for accelerated algorithms.
The work described in the following is joint work with 
Ashia Wilson, Andre Wibisono, and Michael Betancourt
\cite{wibisono2016variational,betancourt2018symplectic}.

Let us equip ourselves with a definition. For any choice of a convex \emph{distance-generating function} $h$, 
we can define the \emph{Bregman divergence} $\mathcal{D}_h$ between
two points as follows:
\begin{equation}
    \mathcal{D}_h(y, x) = h(y) - h(x) - \langle \nabla h(x), y-x \rangle.
    \label{eq:lec2-bregman-distance}
\end{equation}
Note that if $h$ is a quadratic, the Bregman divergence is simply 
the Euclidean distance between $x$ and $y$.

In \cite{wibisono2016variational}, we introduce the \emph{Bregman Lagrangian}:
\begin{equation}
    \mathcal{L}(x,\dot{x},t) = e^{\gamma_t+\alpha_t}
    \left(\mathcal{D}_h\left(x+e^{-\alpha_t}\dot{x}, x\right)
    -e^{\beta_t}f(x)\right),
    \label{eq:lec2-bregman-lagrangian}
\end{equation}
where $\alpha_t, \gamma_t, \beta_t$ are scaling functions, and $f$ is the continuously differentiable convex function
we are interested in optimizing.
Note that we use the shorthand $\alpha_t$ for $\alpha(t)$, and similarly for $\beta_t$ and $\gamma_t$.
By making specific choices of the scaling functions, we will see later that 
we can generate a \emph{family} of algorithms. 
The first term in \eqref{eq:lec2-bregman-lagrangian} is a generalized 
kinetic energy term (note that if we take $h(x)=x^2$, this term reduces to $\dot{x}^2/2$, the familiar kinetic energy from physics).
The second term can be interpreted as a potential energy.

Given a Lagrangian, we can set up the following variational problem over paths:
\begin{equation}
    \min_x \int dt \,\mathcal{L}(x,\dot{x},t).
\end{equation}
The solutions to this optimization problem are precisely the solutions to the
Euler-Lagrange equation
\begin{equation}
    \frac{d}{dt}\left(\frac{\partial}{\partial \dot{x}}\mathcal{L}(x,\dot{x},t)\right)=\frac{\partial}{\partial x}\mathcal{L}(x,\dot{x},t).
\end{equation}

Under what we call the ideal scaling $\dot{\beta}_t\leq e^{\alpha_t}$, $\dot{\gamma}_t=e^{\alpha_t}$, the Euler-Lagrange equation of \eqref{eq:lec2-bregman-lagrangian} is
\begin{equation}
    \label{eq:lec2-bregman-euler-lagrange}
    \ddot{x}_t + \left(e^{\alpha_t}-\dot{\alpha}_t\right)\dot{x}_t + e^{2\alpha_t+\beta_t}\left[\nabla^2 h\left(x_t+e^{-\alpha_t}\dot{x}_t\right)\right]^{-1}\nabla f(x_t)=0.
\end{equation}
We establish the following general convergence result for this dynamics.
\begin{theorem}
The dynamics \eqref{eq:lec2-bregman-euler-lagrange} has the 
convergence rate
\begin{equation}
    f(x_t)-f(x^{\star})\leq\mathcal{O}\left(e^{-\beta_t}\right).
\end{equation}
\end{theorem}
The proof of this result is relatively simple. It involves writing
down a Lyapunov function $\mathcal{E}$ for the dynamics and showing that its time derivative is nonpositive.
\begin{proof}
\begin{align*}
   &\mathcal{E}_t = \mathcal{D}_h\left(x^\star, x_t+e^{-\alpha_t}\dot{x}_t\right)+e^{\beta_t}\left(f(x_t)-f(x^{\star})\right),\\
   &\dot{\mathcal{E}}_t = -e^{\alpha_t+\beta_t}\mathcal{D}_f\left(x^{\star}, x_t\right)
   +e^{\beta_t}\left(\dot{\beta}_t-e^{\alpha_t}\right)(f(x_t)-f(x^{\star})).
\end{align*}
$f$ is convex, and so $\mathcal{D}_f\left(x^{\star}, x_t\right)\geq 0$. The first term in $\dot{\mathcal{E}}_t$ is therefore nonpositive.
Under the ideal scaling, $\dot{\beta}_t-e^{\alpha_t}\leq 0$, and so the second term in $\dot{\mathcal{E}}_t$ is also nonpositive. Thus $\dot{\mathcal{E}}_t\leq 0$.
\end{proof}
With suitable choices of the scaling functions, we can now generate 
a whole family of continuous-time optimization algorithms from
\eqref{eq:lec2-bregman-euler-lagrange} with an accompanying convergence
rate guarantee (in continuous time).
For example, in the Euclidean case, where $h$ is quadratic in $x$, 
if we make the choices $\alpha_t=\log p-\log t$, $\beta_t=p\log t+\log C$, and $\gamma_t=p\log t$ where $C>0$, \eqref{eq:lec2-bregman-euler-lagrange} reduces 
to \eqref{eq:lec2-su-boyd-candes-ode}, the Euler discretization of which
is Nesterov's method.
We can also recover mirror descent, the cubic-regularized Newton method, and
other known discrete-time algorithms from \eqref{eq:lec2-bregman-euler-lagrange},
which serves as a \emph{generator} of accelerated algorithms.
The interested reader may refer to \cite{wibisono2016variational} for details.

Let us now address the problem of mapping back into discrete time.

We begin by noting that the Bregman Lagrangian 
\eqref{eq:lec2-bregman-lagrangian} is a
\emph{covariant} operator. The dynamics generated by covariant operators
has the property that given any two endpoints, the path between them
is fixed. However, we are free to choose $\beta_t$, which controls 
the rate at which this fixed path is traversed.
In continuous time, then, we can obtain arbitrarily fast convergence by
manipulating $\beta_t$, and the algorithmic notion of acceleration that 
we began by wishing to understand \emph{disappears}.
It returns, however, when one considers how to \emph{discretize}
\eqref{eq:lec2-bregman-euler-lagrange}.

It turns out that it is not possible to obtain arbitrarily fast 
\emph{stable} discretizations of \eqref{eq:lec2-bregman-euler-lagrange}.
To understand this, we note that in classical mechanics, 
the dynamics of a system often has conserved
quantities associated with it, such as energy or momentum.
Not all discretization schemes inherit or preserve these conservation laws.
\emph{Symplectic integrators} are a class of discretization schemes that \emph{do}
preserve what is known as phase space volume, and therefore conserve energy, etc.
The method of symplectic integration traces its roots back to 
Jacobi, Hamilton, and Poincar\'{e}, and grew out of the need to discretize
the dynamics of physical systems while preserving the relevant conservation
laws.
Importantly for us, symplectic integrators are provably stable. 
Their stability enables them to take larger step sizes than other kinds of
integrators.
\emph{This} is the origin of acceleration in discrete time.

Symplectic integrators are traditionally defined for time-independent Hamiltonians. (We have thus far been talking about continuous-time algorithms
in a Lagrangian framework, but a simple Legendre transform allows us 
to equivalently frame the discussion in a Hamiltonian framework.)
However, the Bregman Lagrangian \eqref{eq:lec2-bregman-lagrangian} is
time-dependent. Indeed, as we saw, its associated Lyapunov function, which
can be interpreted as an energy, decreases with time. Intuitively, this is
necessary from an optimization standpoint in order to \emph{converge} to 
a point. The question that arises is, how can we use symplectic integrators 
for dissipative systems? In joint work \cite{franca2021} with Guilherme Fran\c{c}a and Ren\'{e} Vidal, we solved this problem by \emph{lifting} the
relevant Hamiltonian to higher dimensions so that it acquires time independence. 
We discretize the time-independent higher-dimensional Hamiltonian, 
and gauge-fix to obtain \emph{rate-preserving} algorithms corresponding to the
original Hamiltonian. 
We call the integrators so-obtained \emph{presymplectic} integrators.
When we use the term ``rate-preserving" to describe an integrator, 
we mean that it preserves the rate of convergence of the continuous-time
dynamics up to an error that can be controlled. 
One of the central results in \cite{franca2021} is that 
pre-symplectic integrators are rate-preserving.
The technical tool used prove this type of result is called \emph{backward error analysis} \cite{Hairer, Hairer2}.

In conclusion, we now have a systematic pipeline for generating accelerated
algorithms with convergence guarantees. The practitioner may (1) ask for 
a certain rate of convergence for a given $f$, (2) write down the corresponding Bregman Hamiltonian, (3) apply a symplectic integrator, and
(4) code up the resulting algorithm.

We mentioned at the beginning of this section that taking a continuous-time
perspective can enable proofs of results we do not yet know how to prove
in discrete time. We give an example of such a result now.
Let us briefly return to the problem of escaping saddles in nonconvex
optimization. We discussed previously that perturbed gradient descent (PGD)
can escape saddles and converge to $\varepsilon$-SOSPs in a number of 
iterations that goes as $1/\varepsilon^2$. This was the content of Theorem \ref{thm:lec2-PGD-convergence}.
In \cite{jin2018-perturbed-accelerated-GD}, we showed that a perturbed version
of \emph{accelerated} gradient descent (PAGD) can do the same in a number 
of iterations that goes as $1/\varepsilon^{7/4}$, i.e., 
it converges \emph{faster} than PGD.
In order to prove this result, we worked in continuous time, studying 
the Hamiltonian that generates PAGD.
The result implies that acceleration can be helpful even in nonconvex
optimization, and not just in the convex case, where all the prior 
discussion of this section has been set.

\paragraph{An open problem.} The continuous-time framework we have
described here only applies to deterministic algorithms. An open 
challenge is to construct a variational framework for diffusions,
which may enable similar illuminating and unifying perspectives on 
accelerated \emph{stochastic} algorithms.

We end this section by briefly
reviewing a few results for stochastic algorithms in continuous time.

\begin{enumerate}
    \item Consider an overdamped Langevin diffusion 
    \begin{equation}
        dx_t = -\nabla U(x_t)\,dt + \sqrt{2}\,dB_t,
        \label{eq:lec2-overdamped-langevin}
    \end{equation}
    where $U:\mathbb{R}^d\rightarrow\mathbb{R}$ is the potential and $B_t$ is 
    standard Brownian motion. 
    The stationary distribution $p^{\star}$ of this diffusion
    is of course the Boltzmann distribution. A discretization of \eqref{eq:lec2-overdamped-langevin}
    is the following Markov Chain Monte Carlo (MCMC) algorithm
    \begin{equation}
        \Tilde{x}_{(k+1)\delta}
        =\Tilde{x}_{k\delta}-\nabla U\left(\Tilde{x}_{k\delta}\right)
        +\sqrt{2\delta}\xi_k,
        \label{eq:lec2-overdamped-langevin-MCMC}
    \end{equation}
    where $k$ is the iteration number, $\delta$ is the time increment,
    and $\xi$ is a $d$-dimensional standard normal random variable.
    A natural question to ask is, how close are we to $p^{\star}$
    after $n$ iterations of the MCMC algorithm?
    Assuming $U(x)$ is $L$-smooth and $\mu$-strongly convex, guarantees of the 
    following type have been proven: if $n\geq\mathcal{O}(d/\varepsilon^2)$, 
    then $D\left(p^{(n)},p^{\star}\right)\leq\varepsilon$, where $D$
    is the relevant distance measure, and $p^{(n)}$ is the probability density
    generated by the MCMC algorithm at the $n^{th}$ iteration. 
    Results exist in the cases where
    $D$ is the total variation distance \cite{dalalyan2014}, the $2$-Wasserstein distance \cite{durmus-moulines}, and the Kullback-Leibler divergence \cite{cheng-bartlett}.
    
    \item Now consider an underdamped Langevin diffusion
    \begin{align}
        &dx_t = v_t \, dt,\nonumber\\
        &dv_t = -\gamma v_t \, dt + \lambda\nabla U(x_t) \, dt + \sqrt{2\gamma\lambda} \, dB_t,
        \label{eq:lec2-underdamped-langevin}
    \end{align}
    where $\gamma$ is the coefficient of friction.
    The stationary distribution $p^{\star}$ is proportional to \\
    $\exp\left(-U(x)-||v||^2/2\lambda\right)$.
    In \cite{cheng-underdamped-langevin}, we are able to make a similar
    guarantee on the MCMC algorithm corresponding to \eqref{eq:lec2-underdamped-langevin} as for \eqref{eq:lec2-overdamped-langevin-MCMC}. 
    That is, assuming $U(x)$ is $L$-smooth and $\mu$-strongly convex, 
    if $n\geq\mathcal{O}(\sqrt{d}/\varepsilon)$, 
    then $D\left(p^{(n)},p^{\star}\right)\leq\varepsilon$, where $D$
    is the $2$-Wasserstein distance.
    This is a faster rate than in the overdamped case, where we needed
    the square of the number of iterations required here to make the
    same guarantee on the distance between $p^{(n)}$ and $p^{\star}$.
    The intuition behind this faster rate lies in the fact that
    underdamped diffusions are smoother than overdamped ones,
    and therefore easier to discretize.
    
    The proof technique in \cite{cheng-underdamped-langevin} involves
    another coupling time argument. We saw one of these in the discussion
    on how PGD escapes saddle points. In this case we study the coupling
    time of two Brownian motions that are reflections of each other.
    
    \item An underdamped Langevin diffusion corresponds to Nesterov
    acceleration in the simplex of probability measures \cite{ma2019}.
    \item Higher-order Langevin diffusion yields an accelerated MCMC algorithm \cite{mou2021}.
    \item Smoothness isn't necessary for Langevin diffusion \cite{chatterji2020}.
    \item Langevin-based algorithms can achieve logarithmic regret on multi-armed bandits \cite{mazumdar20a}.
\end{enumerate}

Now that we have some context in the results just stated, we can make
precise the type of open problem mentioned earlier. Suppose we wish to 
diffuse (under some stochastic differential equation) to a target
distribution $p^{\star}$. What equation gets us to $p^{\star}$ the fastest, 
and what is the correct mathematical formulation for this problem?

\subsection{Variational Inequalities: From Minima to Nash Equilibria and Fixed Points}
\label{sec:lec2-finding-nash-equilibria}

In this last section of the lecture, we will introduce the setting
for the third lecture.
From now on, we will be concerned with how to find (Nash) equilibria of certain types of 
functions $f$ (we will make this precise). The discussion will break
with all the technical material we have seen so far in an important way: Nash 
equilibria are \emph{saddle points} of $f$. Therefore, instead of 
studying algorithms that avoid saddle points, we will talk about 
algorithms that \emph{converge} to saddle points. In the next lecture,
we will study two such algorithms, the proximal point method, and the
extragradient method.

In the following, we introduce (1) zero-sum games, 
(2) Nash equilibria, and (3) variational inequalities.
We then give some context for why we introduce these ideas, and for
why we are interested in computing Nash equilibria.

\subsubsection{Two-Player Zero-Sum Games}

Consider a function $f:\mathbb{R}^{d_1}\times\mathbb{R}^{d_2}\rightarrow\mathbb{R}$, 
involving two ``players'' $x_1$ and $x_2$. One player acts
to minimize $f$, and the other acts to maximize it.
This leads to the following optimization problem, also called a 
two-player zero-sum \emph{game}:
\begin{equation}
    \min_{x_1\in\mathbb{R}^{d_1}}\max_{x_2\in\mathbb{R}^{d_2}} f(x_1, x_2).
    \label{eq:lec2-zero-sum-game}
\end{equation}
The game is called ``zero-sum" because the cost functions of players 
$x_1$ and $x_2$, respectively $f$ and $-f$, add to zero.
A simple instance of \eqref{eq:lec2-zero-sum-game} is the 
zero-sum bilinear game $f(x_1, x_2)=x_1^{\top}A x_2$, 
where $A$ is a positive semi-definite symmetric matrix.
\eqref{eq:lec2-zero-sum-game} is a member of the set of \emph{convex-concave} games, 
called so since $f$ is convex in $x_1$ and concave in $x_2$. 
Convex-concave games are in turn a subset of the family of \emph{minimax} games, 
which are a subset of the family of \emph{monotone} games. 
This last piece of terminology and its significance will be explained Section 
\ref{sec:lec3-monotone-operators}.

The solutions of \eqref{eq:lec2-zero-sum-game} are saddle 
points (and \emph{not} minima) of $f$; they are the set of points
$\left(x_{1}^{\star}, x_2^{\star}\right)$ such that
\begin{equation}
    f\left(x_{1}^{\star}, x_2\right)\leq f\left(x_{1}^{\star}, x_{2}^{\star}\right)
    \leq f\left(x_{1}, x_{2}^{\star}\right).
\end{equation}
In the parlance of game theory, some (but not, in general, all) of the 
solutions of the optimization problem that defines the dynamics of a game 
are called \emph{equilibria}.

We have already encountered one example of a (Stackelberg) game in 
Section \ref{sec:lec1-strategic-classification}. 
Stackelberg games are sequential (players take turns to play).
\eqref{eq:lec2-zero-sum-game} is not a sequential game.
Its equilibria are called \emph{Nash equilibria}.
We will define a Nash equilibrium for an $N$-player game formally later.
Intuitively, a Nash equilibrium is a saddle point of $f$ that 
neither player is incentivized to move away from.
When we introduce $N$-player games,
this intuition will carry over: Nash equilibria will correspond
to points where each player is as happy as possible 
\emph{relative to what all the other players are doing}.

\subsubsection{Variational Inequalities}

Two-player zero-sum games (and indeed, the larger class of monotone games) can be written as instances of a general 
class of optimization problems called \emph{variational inequalities} (VIs).

\begin{definition}[Variational Inequality]
Given a subset $\mathcal{X}$ of $\mathbb{R}^d$ and an operator (or vector field)
$F:\mathcal{X}\rightarrow\mathbb{R}^d$,
a \emph{variational inequality} is an optimization problem of the 
following form. Find $x^{\star}$ such that 
\begin{equation}
    \langle F(x^{\star}), \, x-x^{\star}\rangle \geq 0 \,\,\, \forall x\in\mathcal{X}.
    \label{eq:lec2-VI-def}
\end{equation}
\end{definition}
Note that $x^{\star}$ need not always be unique, and that $\mathcal{X}$ 
may or may not be convex. (There are certain requirements on the topology 
of $\mathcal{X}$ but we will omit discussion of these here.)
Variational inequalities unify the formulations of unconstrained 
and constrained optimization problems, as well optimization problems 
involving equilibria, and many others.
If we replace the inequality with equality in \eqref{eq:lec2-VI-def}, 
and $F$ is a gradient field, then we recover classical optimization; 
furthermore, if $\mathcal{X}$ is (a proper subset of) $\mathbb{R}^n$, 
then the optimization problem is (constrained) unconstrained.

In the language of variational inequalities, the two-player zero-sum
game is represented by the vector field
\begin{equation}
    F(x)
    =
    \begin{pmatrix}
    \nabla_{x_1}f(x_1, x_2)\\
    -\nabla_{x_2}f(x_1, x_2)
    \end{pmatrix},
    \label{eq:lec2-two-player-game-vector-field}
\end{equation}
with $\mathcal{X}=\mathcal{X}_1\otimes\mathcal{X}_2$.
Note that $F$ is \emph{not} the gradient of an underlying function.

Representing a game as a VI allows us to express the problem
of computing the equilibria of the game as the problem of computing the fixed points of 
the relevant operator $F$. Although on the face of it this mathematical mapping may not appear 
to confer an advantage, in fact there exists a vast machinery to compute the fixed points of
nonlinear operators that we can take advantage of through it.
We will come back to this point in the next section.

\subsubsection{Nash Equilibria}

Finally, let us define the notion of a Nash equilibrium. 
Consider a game 
of
$N$ players $x_i$.
Let each $x_i\in\mathbb{R}^{d_i}$. 
Stacking all the players into a vector gives the
\emph{strategy vector} $x$, which lives in $\mathbb{R}^{\sum_{i=1}^{N}d_i}$. 
Each player has an associated cost function
$g_i(x_i,x_{-i})$, where we use $x_{-i}$ to denote the $\sum_{j\neq i}d_j-$ dimensional 
vector representing all players except $x_i$.
The \emph{solution} for player $i$ lives in the \emph{solution set} 
$S_i(x_{-i})$, which is the set of solutions to the optimization problem 
$\min_{x_i} g_i(x_i, x_{-i})$
such that $x_i\in K_i\subseteq\mathbb{R}^{d_i}$ where $K_i$ is a constraint set.

\begin{definition}[Nash Equilibrium]
A \emph{Nash equilibrium} is a vector $x^{\star}\in\mathbb{R}^{\sum_{i=1}^{N}d_i}$ s.t.
$x_{i}^{\star}\in S_i\left(x_{-i}^{\star}\right)$ $\forall i$.
\end{definition}

The problem of finding a Nash equilibrium can be written as a variational inequality with
\begin{equation}
\mathcal{X}=\mathcal{K}_1\otimes\dots\otimes\mathcal{K}_N,\,\,\,
    F(x)
    =
    \begin{pmatrix}
    \nabla_{x_1}g_1(x)\\
    \vdots\\
    \nabla_{x_N}g_N(x)
    \end{pmatrix},
    \label{eq:lec2-variational-formulation-nash}
\end{equation}
if and only if each $K_i$ is a closed convex subset of $\mathbb{R}^{d_i}$
and each $g_i$ is a convex function of $x_i$.
See \cite{facchinei-pang}, Proposition 1.4.2, for a proof of this fact.

Exercise to the reader: check that if $N=2$, $g_1=f$, and $g_2=-f$, then \eqref{eq:lec2-variational-formulation-nash}
reduces to a two-player zero sum game.

Nash equilibria are not the only equilibria that can be represented
as solutions of variational inequalities.
Some other interesting examples, which we will not address here, are
Markov perfect equilibria (from probability theory), Walrasian equilibria (from economics), 
and the equilibria of frictional contact problems (from mechanics).

Armed with the definitions of games and Nash equilibria, we can
briefly mentioning a gradient-based method developed
in collaboration with Eric Mazumdar and S. Shankar Sastry for finding
local Nash equilibria in two-player zero-sum games \cite{mazumdar2019-on-local-nash}.
This example also pulls together
some of the ideas we saw in the last section, namely of working in 
continuous time and paying attention to symplectic geometry.
Consider the problem \eqref{eq:lec2-zero-sum-game}.
We have mentioned that the solutions to this optimization problem
are saddle points of $f$. However, not \emph{all} saddle points of $f$ 
are Nash equilibria. The latter are
specifically saddle points that are \emph{axis-aligned}; we refer the 
reader to the paper \cite{mazumdar2019-on-local-nash} for further details
on what this means.
Prior to our work, 
it was known that gradient-based algorithms for finding Nash equilibria could
(in the worst case) diverge, or get stuck in limit cycles (i.e., simply
follow the flow of the vector field). We showed that in addition, 
these algorithms could also converge to non-Nash fixed points.
The distinction between non-Nash and Nash fixed points
has to do with the symplectic component of the vector field. 
When the symplectic component of the vector field near a fixed point 
is removed, only Nash fixed points persist. 
We exploited this fact and developed an algorithm we called ``local symplectic
surgery" that could effect this removal, and therefore provably converge
only to Nash equilibria.

\subsubsection{Computing Nash Equilibria}

Why are we interested in computing Nash equilibria?
And why do we bring up the framework of variational inequalities 
in this context? We explain.

Nash equilibria are hard to compute in general 
\cite{daskalakis2009-complexity-Nash}.
This fact has profound implications for the mathematical study of markets.
Much effort in microeconomics has been made to
understand the properties of Nash equilibria under the assumption
that the game/system/market under consideration is already in such an 
equilibrium.
For this assumption to hold, the players must be able to arrive at/compute
the equilibrium.
However, the negative computability result of \cite{daskalakis2009-complexity-Nash}
implies that this is generally not possible.
The questions then are, what are the classes of
games in which it \emph{is} tractable to compute Nash equilibria,
and what are the algorithms that are needed for this purpose?

It turns out that the class of two-player zero-sum games is tractable
in this specific sense. We have already seen that as long as the cost
function is convex, two-player zero-sum games can be written as 
variational inequalities, and in particular, variational inequalities
involving \emph{monotone} operators. We can therefore apply and build on mathematical
machinery originally developed to compute the fixed points of monotone
operators to compute the Nash equilibria of two-player
zero-sum games.

We will define monotonicity in the
next lecture. For now, it suffices to note that monotone operators are 
a class of operators for which it is
guaranteed that a fixed point exists, or equivalently, it is 
a class of operators
for which it is possible to guarantee that the corresponding variational 
inequalities actually have solutions.

Another reason why the language of variational inequalities is useful to
us is that it naturally centralizes the role of dynamics in understanding
the behavior of interacting agents. As we have seen in \cite{zrnic2021leads},
and in our discussion of nonconvex optimization,
the dynamics \emph{influences which fixed point we arrive at},
and so it is important to understand the relationship between 
a fixed-point-finding algorithm and the properties of the fixed points
it converges to. This is helpful both to understand the behaviors of existing
markets and to effectively design new markets with certain properties
(and not others) for the future.

The literature on fixed-point-finding algorithms for monotone operators
is vast. To simplify the exposition in what follows, and to make
good use of the previous discussion, we will take a relatively unusual
approach. In the next lecture, we will first establish that 
gradient descent (which does well at finding the minima of convex
functions) generally diverges on monotone problems. 
We will identify the reason why, and then ask how we can ``fix"
it. This will motivate the discussion of the proximal
point method and the extragradient algorithm.

\pagebreak

\section{Computing Equilibria}

Consider the two kinds of vector fields pictured in 
Fig.\,\ref{fig:lec3-fixed-points}. In Fig.\,\ref{fig:lec3-fixed-points}(a),
the flow is toward a point; our discussion of convex optimization in
Lectures 1 and 2 was in this context.
In Fig.\,\ref{fig:lec3-fixed-points}(b) the flow moves \emph{around} a point.
Both the flows pictured in Fig.\,\ref{fig:lec3-fixed-points}
are generated by monotone operators. 
Given a monotone game, it will not a priori be apparent which of the two
types of flows we are dealing with. To reliably compute equilibria,
therefore, we will need algorithms that handle both cases. 
We will spend the rest of these notes
developing and studying such algorithms.

\subsection{Monotone Operators}
\label{sec:lec3-monotone-operators}

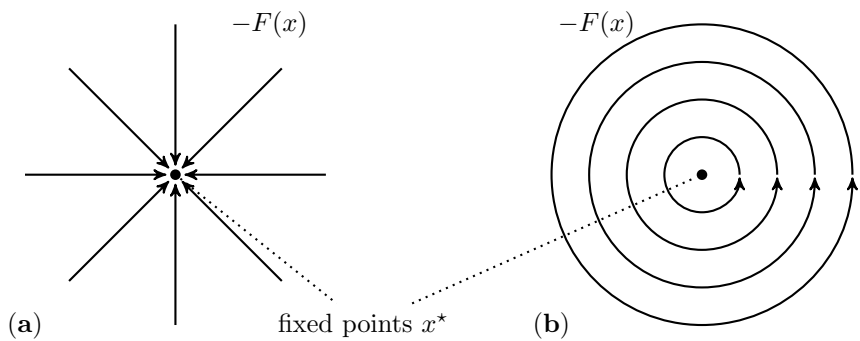
\begin{figure}[!h]
    \centering
\begin{tikzpicture}[>=stealth',shorten >=1pt,auto,node distance=3cm,
  thick,main node/.style={circle,fill=black,minimum size=4pt,inner sep=0pt}]

  % define the two points
  \node[main node] (1) at (0,0) {};
  \node[main node] (2) at (7,0) {};

  % draw converging arrows for the point on the left
  \foreach \angle in {0,45,...,315}
    \draw[->] (\angle:2cm) -- (1);

  % draw circles around the point on the right
  \foreach \radius in {0.5,1,1.5,2}
    \draw[->] (2) ++(\radius,0) arc (0:360:\radius cm);

  % draw dotted lines connecting the two points to a text saying "fixed points"
  \node (label) at (2.5,-2) {fixed points $x^{\star}$};
  \draw[dotted] (1) -- (label) -- (2);

  % label the vector fields
  \node at (1.25,2) {$-F(x)$};
  \node at (-2, -2) {$\mathbf{(a)}$};
  \node at (5.6,2) {$-F(x)$};
  \node at (5, -2) {$\mathbf{(b)}$};

\end{tikzpicture}
    \caption{The notion of a fixed point is broader than that of a minimum. Here we have examples of two types of vector fields $F(x)$. In (a), the negative flow (indicated by arrow heads) of the field is toward the fixed point; this fixed point could be a minimum. In (b), the negative flow of the vector field is \emph{around} the fixed point; this fixed point is not a minimum.
    In this lecture, we study algorithms that compute the fixed points of vector fields associated with monotone operators, which can be either of the flavor (a) \emph{or} (b).
    }
    \label{fig:lec3-fixed-points}
\end{figure}

To begin the discussion, we introduce the notion of monotonicity, which 
generalizes the notion of convexity, and a strengthening of this notion 
called \emph{strong monotonocity}.
\begin{definition}[Monotone Operator]
An operator $F$ is said to be \emph{monotone} on a set $\mathcal{X}$ if
\begin{equation}
    \langle F(x)-F(y), x-y \rangle \geq 0 \,\,\forall x, y \in\mathcal{X}.
    \label{eq:lec3-monotonicity}
\end{equation}
The property \eqref{eq:lec3-monotonicity} is known as \emph{monotonicity}.
\end{definition}
We will assume that the constraint set $\mathcal{X}$ is convex. 
When this is the case, we are guaranteed
that a fixed point of $F$ exists.

We remind the reader that our notation for a fixed point of $F$ is $x^{\star}$.
Taking $x=x^{\star}$ and $y=x$ and setting $F(x^{\star})=0$ in \eqref{eq:lec3-monotonicity},
we see that monotonicity of $F$ implies that the angle $\theta$ between
the vectors $-F(x)$ and $x^{\star}-x$ for an arbitrary $x$ can be no greater than $\pi/2$.
This is illustrated in Fig.\,\ref{fig:lec3-monotonicity}.

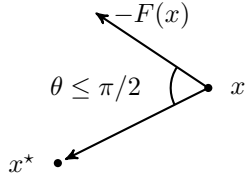
\begin{figure}[!h]
\centering
\begin{tikzpicture}[>=stealth',thick,inner sep=8pt]

  % define the points
  \coordinate[label=right:$x$] (x) at (2,0);
  \draw [fill] (x) circle (.04);
  \coordinate[label=left:$x^{\star}$] (xstar) at (0,-1);
  \draw [fill] (xstar) circle (.04);
  \coordinate (y) at (0.5,1);

  % draw the arrows
  \draw[->,shorten >=3pt] (x) -- (xstar);
  \draw[->] (x) -- node[midway,above] {$-F(x)$} (y);

  % draw and label the angle
  \pic[draw=black, angle eccentricity=1.2, angle radius=0.5 cm] {angle=y--x--xstar};
  % old command
  % \pic["$\theta", draw=black, angle eccentricity=1.2, angle radius=0.5 cm] {angle=y--x--xstar};

  % add the inequality
  \node at (0.5,0) {$\theta \leq \pi/2$};

\end{tikzpicture}
    \caption{
    Monotonicity implies that the angle $\theta$ between
    the operator $-F(x)$ and the vector pointing from $x$
    to $x^{\star}$ is at most $\pi/2$ for all $x$. To see this in 
    \eqref{eq:lec3-monotonicity}, take $x=x^{\star}$, $y=x$, and
    without loss of generality let $F(x^{\star})=0$.
    }
    \label{fig:lec3-monotonicity}
\end{figure}

\begin{definition}[Strong Monotonicity]
An operator $F$ is said to be \emph{strongly monotone} on a set $\mathcal{X}$ if
\begin{equation}
    \langle F(x)-F(y), x-y \rangle \geq \mu||x-y||^2 \,\,\forall x, y \in\mathcal{X},
    \label{eq:lec3-strong-monotonicity}
\end{equation}
for some $\mu>0$.
\end{definition}
Strong monotonicity is a strengthening of monotonicity. It guarantees
that the angle between $-F(x)$ and $x^{\star}-x$ is \emph{strictly} less than $\pi/2$.
Furthermore, given any arbitrary point $x$, strong monotonicity of $F$ ensures that
the fixed point lies on or above the quadratic function $\mu||x-y||^2$ centered at $x$.
In contrast, monotonicity merely implies that $x^{\star}$ is on the same side of the
hyperplane defined by $-F$ at $x$. This is illustrated in Fig.\,\ref{fig:lec3-strong-monotonicity}.

\begin{figure}[!h]
\vspace{0.2in}
    \centering
\begin{tikzpicture}
     \draw[purple] (-5,4) parabola bend (0,0) (5,4);
     \node[purple] at (4.5,2) {$\mu||x-y||^2$};
     \draw[purple,line width=0.4mm,->] (0,0) node[black,anchor=north]{$x$} -- (-2,3) node[purple,anchor=south] {$x^{\star}_{SM}$};
     \draw[line width=0.4mm,->] (0,0) -- (0,2) node[anchor=south]{} node[midway, anchor=west]{$-F(x)$};
     \draw[cyan,line width=0.4mm,->] (0,0) -- (3,0.3) node[anchor=south] {$~~~x^{\star}_M$};
     \draw (2,4) node[purple,anchor=north]{Strong Monotonicity};
     \draw (4,1.3) node[cyan,anchor=north]{Monotonicity};
     \draw[black,->] (0,0) -- (3,0) node[right] {$y$};
\end{tikzpicture}
    \caption{While monotonicity
    of $F$ guarantees that $x^{\star}$ is on the same side of the
    hyperplane defined by $-F$ at $x$ as the quadratic 
    $\mu||x-y||^2$, strong monotonicity additionally guarantees that 
    $x^{\star}$ is on or above the quadratic.
    We use the label $x^{\star}_M$ ($x^{\star}_{SM}$) to indicate a possible fixed point of $F$ when
    $F$ is monotone (strongly monotone). In general, $x^{\star}_M$ may be either aligned with 
    the axis $y$ or anywhere above it, while $x^{\star}_{SM}$ is confined to be on or above 
    $\mu||x-y||^2$.
    }
    \label{fig:lec3-strong-monotonicity}
\end{figure}
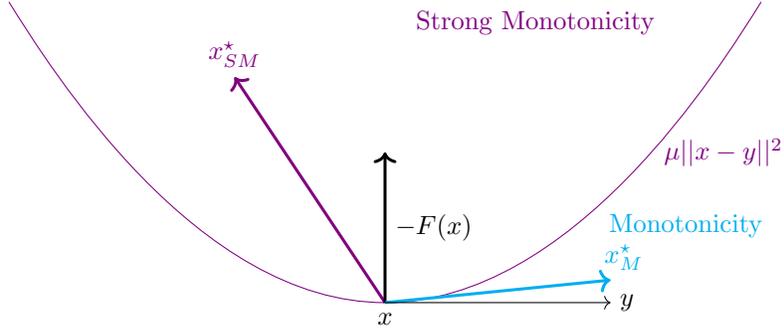

\subsection{Fixed-Point Finding Algorithms}
\label{sec:lec3-fixed-point-algorithms}

\subsubsection{A Naive Algorithm}

We will now see that a (naive) gradient-descent-like algorithm diverges
on monotone problems. The cause of the divergent behavior is that
monotonicity allows the angle $\theta$ between $-F(x)$ and $x^{\star}-x$
to be as large as $\pi/2$
(see Fig.\,\ref{fig:lec3-monotonicity}). 
For problems of the type in Fig.\,\ref{fig:lec3-fixed-points}(b), 
this implies that
the algorithm may simply follow the flow of $F$ around $x^{\star}$ and not
make progress towards it, or it may even diverge.  As the reader may have guessed, we will be
able to fix this problem if we assume that $F$ is strongly monotone, ensuring
that $\theta$ is strictly bounded away from $\pi/2$.

Using the notation $\proj_{\mathcal{X}}(x)$ for the projection 
of the vector $x$ onto the convex constraint set
$\mathcal{X}$, let us study the behavior of the iteration
\begin{equation}
    x_{k+1} = \proj_{\mathcal{X}}{\left(x_k-\eta F(x_k)\right)}.
    \label{eq:lec3-naive-alg}
\end{equation}
Does \eqref{eq:lec3-naive-alg} converge to $x^{\star}$?

\begin{figure}[!h]
    \centering
\begin{tikzpicture}[>=stealth',thick]

  % draw the convex set
  \draw (2,0) arc (0:180:2cm) node[midway,below] {$\mathcal{X}$};

  % define the points and their projections
  \coordinate[label=above right:$x$] (x) at (3,3);
  \draw [fill] (x) circle (.04);
  \coordinate[label=above left:$y$] (y) at (-3,3);
  \draw [fill] (y) circle (.04);
  \coordinate (px) at (1.414,1.414);
  \draw [fill] (px) circle (.04);
  \coordinate (py) at (-1.414,1.414);
  \draw [fill] (py) circle (.04);
  
  % label the projections
  \node[above right = -0.2 cm and 0.3 cm of px] {$\proj_{\mathcal{X}}(x)$};
  \node[above left = -0.3 cm and 0.2 cm of py] {$\proj_{\mathcal{X}}(y)$};

  % draw the arrows from the points to their projections
  \draw[->, shorten >= 3pt] (x) -- (px);
  \draw[->, shorten >= 3pt] (y) -- (py);

\end{tikzpicture}
    \caption{Projection onto a convex set $\mathcal{X}$ is a contractive
    operation, which means that the difference of projections is at most
    the difference of the arguments. 
    }
    \label{fig:lec3-projection-on-convexset}
\end{figure}
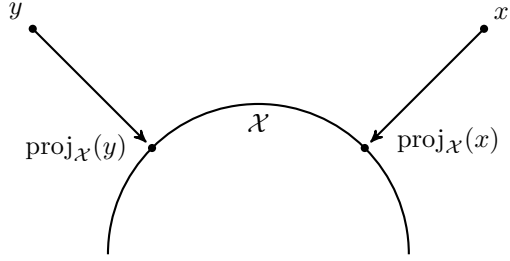

As will be familiar by now, we begin by looking at the distance
$||x_{k+1}-x^{\star}||$:
\begin{equation}
    \left|\left|x_{k+1} - x^{\star}\right|\right|^2 
    =\left|\left|\proj_{\mathcal{X}}\left(x_k-\eta F(x_k)\right) - 
    \proj_{\mathcal{X}}\left(x^{\star}-\eta F(x^{\star})\right)\right|\right|^2.
    \label{eq:lec3-naive-iter-difference-of-projections}
\end{equation}
Projection onto a convex set is a contractive operation.\footnote{In fact, projection onto a convex set is a \emph{firmly non-expansive} procedure. We will formally define firm non-expansivity later in this lecture.}
This implies that difference of two projections is at most the difference 
of the arguments
(see Fig.\,\ref{fig:lec3-projection-on-convexset} for a pictorial proof).
We can use this fact to drop the projection operator 
in \eqref{eq:lec3-naive-iter-difference-of-projections} at the expense of an inequality.
Doing so, and expanding the resulting squared norm, we have
\begin{align}
    \left|\left|x_{k+1} - x^{\star}\right|\right|^2
    &\leq
    \left|\left|x_k-\eta F(x_k)-x^{\star}+\eta F(x^{\star})\right|\right|^2
    \nonumber\\
    &=\left|\left|x_k-x^{\star}\right|\right|^2
    -2\eta\langle F(x_k)-F(x^{\star}), \,\,\,x_k-x^{\star}\rangle 
    +\eta^2\left|\left|F(x_k)-F(x^{\star})\right|\right|^2.
    \label{eq:lec3-naive-iter-intermediate}
\end{align}
Applying the monotonicity property \eqref{eq:lec3-monotonicity}
to the second term in \eqref{eq:lec3-naive-iter-intermediate}, we arrive at
\begin{equation}
    \left|\left|x_{k+1} - x^{\star}\right|\right|^2
    \leq
    \left|\left|x_k-x^{\star}\right|\right|^2
    +\eta^2\left|\left|F(x_k)-F(x^{\star})\right|\right|^2.
    \label{eq:lec3-naive-iter-diverges}
\end{equation}
Since both the terms on the right-hand side of \eqref{eq:lec3-naive-iter-diverges} are positive, we cannot conclude that $||x_{k+1}-x^{\star}||$
is shrinking with iteration number.
Thus, the naive gradient-descent-like iteration \eqref{eq:lec3-naive-alg} 
may diverge on generally monotone problems.

As we have already mentioned, we can fix this by assuming that $F$ is not
just monotone but strongly monotone. This allows us to retain the second term
in \eqref{eq:lec3-naive-iter-intermediate}.
It is important to retain this term, since it is the only one
with a negative sign in \eqref{eq:lec3-naive-iter-intermediate}, and is therefore
needed to counteract the expansive effect of the third term.
We will also assume that $F$ is $L$-Lipschitz.
With these two assumptions, we can easily show per-iteration contraction of the 
distance between the iterate and the fixed point.
Starting once again from \eqref{eq:lec3-naive-iter-intermediate}, we have
\begin{align}
\left|\left|x_{k+1} - x^{\star}\right|\right|^2 
    &\leq\left|\left|x_k-x^{\star}\right|\right|^2
    -2\eta\langle F(x_k)-F(x^{\star}), \,\,\,x_k-x^{\star}\rangle 
    +\eta^2\left|\left|F(x_k)-F(x^{\star})\right|\right|^2
    \nonumber\\
    &\leq\left|\left|x_k-x^\star\right|\right|^2
    -2\eta\mu\left|\left|x_k-x^{\star}\right|\right|^2
    +\eta^2 L^2||x_k - x^\star||^2
    \nonumber\\
    &=\left(1-2\eta\mu+\eta^2L^2\right)||x_k - x^\star||^2
    \nonumber\\
    &\overset{\eta=\mu/L^2}{=}\left(1-\frac{\mu^2}{L^2}\right)||x_k-x^{\star}||^2.
    \label{eq:lec3-naive-iter-convergence}
\end{align}

With \eqref{eq:lec3-naive-iter-convergence}, we have proved the following theorem.
\begin{theorem}
\label{thm:naive-iter-strongly-monotone-F}
For a $\mu$-strongly monotone and $L$-Lipschitz operator $F$ and a convex
constraint set $\mathcal{X}$,
the iterates of the gradient-descent-like algorithm 
\eqref{eq:lec3-naive-alg} with step size $\mu/L^2$ converge 
to the fixed point $x^{\star}$ of $F$ with the rate $1-(\mu/L)^2$.
\end{theorem}

The rate of convergence in Theorem \ref{thm:naive-iter-strongly-monotone-F}
is slower than the (roughly) $1-\mu/L$ rate achieved by gradient descent on 
strongly convex functions in Theorem \ref{thm:GD-on-strongly-convex-f}.
In fact, it is possible to extract this faster rate from \eqref{eq:lec3-naive-alg}
by making one further assumption on $F$, that of \emph{co-coercivity}. 
We give its definition next.

\begin{definition}[Co-coercivity]
An operator $F$ is said to be \emph{co-coercive} on a set $\mathcal{X}$ if
there exists an $\alpha>0$ such that
\begin{equation}
    \langle F(x)-F(y), x-y \rangle \geq 
    \frac{1}{\alpha}||F(x)-F(y)||^2 \,\,\,\,\forall x, y \in\mathcal{X}.
    \label{eq:lec3-cocoercivity}
\end{equation}
\end{definition}

In fact, a $\mu$-strongly monotone and $L$-Lipschitz operator is co-coercive with $\alpha=L$, but $L$ is a suboptimal value of $\alpha$.
It is possible to achieve a faster rate by identifying the optimal value of $\alpha$.
In particular, if $F$ is $\mu$-strongly monotone and $\alpha$-co-coercive, it is possible to achieve a convergence rate of $1-\mu/\alpha$. 
Lastly, we note that if $F$ is the gradient field $\nabla f$ of a strongly convex and $L$-smooth function $f$, then it is $L$-co-coercive.
This fact and the convergence result stated just before it together recover the fast convergence rate of gradient descent from Theorem \ref{thm:GD-on-strongly-convex-f}
in the variational inequality framework.

\subsubsection{The Proximal Point Method}

Let us take a second look at \eqref{eq:lec3-naive-iter-diverges}, 
which implies that if $F$ is monotone, then applying an iteration of 
\eqref{eq:lec3-naive-alg} can in general only increase the distance to the optimum. In the special case where
$F$ is the gradient field of a convex and $L$-smooth function, however, 
\eqref{eq:lec3-naive-alg} reduces to gradient descent, which \emph{does} converge. In other words,
\eqref{eq:lec3-naive-alg} diverges for general 
monotone vector fields 
(such as the example in Fig.\,\ref{fig:lec3-fixed-points}(b)), 
but not in the special case of a gradient field 
(Fig.\,\ref{fig:lec3-fixed-points}(a)).
We will now briefly introduce a modification to 
\eqref{eq:lec3-naive-alg} that we will show causes it to diverge on 
\emph{all} monotone problems. We will then fix this algorithm to 
obtain another algorithm---the proximal point method---that converges 
for all monotone problems.

Let us consider the update
\begin{equation}
    x_{k+1}=x_k+\eta F(x_k)=\left(I+\eta F\right)x_{k},
    \label{eq:lec3-generally-divergent-update}
\end{equation}
where $\eta >0$ is the step size, $I$ is the identity operator and 
$F$ is a monotone operator.
A simple calculation shows that the operator $\left(I+\eta F\right)$ is 
expansive:
\begin{align}
    \left|\left|\left(I+\eta F\right)(x)-\left(I+\eta F\right)(y)\right|\right|^2
    &=\left|\left|x-y\right|\right|^2 +2\eta\langle F(x)-F(y), x-y\rangle + \eta^2\left|\left|F(x)-F(y)\right|\right|^2
    \nonumber\\
    &\geq \left|\left|x-y\right|\right|^2 + \eta^2\left|\left|F(x)-F(y)\right|\right|^2.
\end{align}
Applying $\left(I+\eta F\right)$ to two points $x$ and $y$ thus strictly 
increases the squared distance between them. This implies that 
\eqref{eq:lec3-generally-divergent-update} always diverges.
This can also be seen simply by inspection in Fig.\,\ref{fig:lec3-fixed-points}: following the flow $F$ (as opposed to $-F$, which is what gradient descent does) is a bad strategy both when $F$ is a gradient field and when it isn't.

We can use the expansive property of $\left(I+\eta F\right)$ to our advantage.
In particular, notice that if we run 
\eqref{eq:lec3-generally-divergent-update} 
\emph{backward in time},
\begin{equation}
    x_k=\left(I+\eta F\right)^{-1}x_{k+1},
    \label{eq:lec3-run-it-backwards}
\end{equation}
then we might expect to obtain a procedure that shrinks
the distance between any two points instead of increasing it.
The \emph{proximal point method}, which we now introduce, 
does precisely this. 
The proximal point update is given by rearranging \eqref{eq:lec3-run-it-backwards} as follows:
\begin{equation}
    x_{k+1}=x_k-\eta F(x_{k+1}).
    \label{eq:lec3-prox-point-update}
\end{equation}
In the form that we present it here, it was first studied by
Rockafellar in 1976 \cite{rockafella1976r-proximal-point}.

\begin{figure}[!h]
    \centering
    \begin{tikzpicture}
        \draw[line width=0.5mm,->,shorten >= 2pt] (0,0) -- (2,2) node[anchor=west] {$x_k$};
        \draw [fill] (2,2) circle (.05);
        \draw[line width=0.5mm,->,shorten >= 2pt] (0,0) -- (3,0) node[anchor=west] {$x_{k+1}=B(x_k)$};
        \draw [fill] (3,0) circle (.05);
        \draw[line width=0.5mm,->,shorten >= 2pt] (0,0) -- (4,-2) node[anchor=west] {$(2B-I)x_k$};
        \draw [fill] (4,-2) circle (.05);
        \draw[line width=0.5mm,-] (2,2) -- (4,-2);
    \end{tikzpicture}
    \caption{Applying the backward operator to $x_{k}$ brings us to the point $x_{k+1}$. This is the proximal point update. Applying the operator $2B-I$ to $x_k$ results in a reflection about $x_{k+1}$.}
    \label{fig:firmly}
\end{figure}
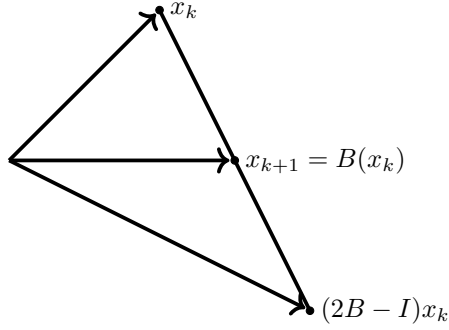

\begin{algorithm}
\caption{Proximal Point Method}
\label{alg:lec3-proximal-point}
\begin{algorithmic}
\Require $\eta>0$\\
\State $x_{k+1}=x_k-\eta F(x_{k+1})$
\end{algorithmic}
\end{algorithm}

Notice that the proximal point update given in
\eqref{eq:lec3-prox-point-update} is an \emph{implicit} equation, i.e.,
it contains $x_{k+1}$ on both sides, and must be solved 
at each iteration.
Since there are multiple ways to solve the implicit equation,
the proximal point method as stated defines a \emph{family} of 
algorithms.

Before we enter the discussion of ways to solve for $x_{k+1}$,
let us study the effect on $||x_{k+1}-x^{\star}||^2$ of applying a
single step of the proximal point method to confirm the intuition
that led us to introduce it. For simplicity, and without loss of generality,
we set $F(x^{\star})$ to zero for the remainder of this lecture.
\begin{align}
    \left|\left|x_{k+1}-x^{\star}\right|\right|^2
    &=\left|\left|x_k-\eta F(x_{k+1})-x^{\star}\right|\right|^2
    \nonumber\\
    &= \left|\left|x_k-x^{\star}\right|\right|^2
    -2\eta\langle F(x_{k+1}), x_k-x^{\star}\rangle
    +\eta^2\left|\left|F(x_{k+1})\right|\right|^2
    \nonumber\\
    &\overset{(a)}{=} \left|\left|x_k-x^{\star}\right|\right|^2
    -2\eta\langle F(x_{k+1}), x_{k+1}-x^{\star}\rangle
    -2\eta\langle F(x_{k+1}), x_k-x_{k+1}\rangle
    +\eta^2\left|\left|F(x_{k+1})\right|\right|^2
    \nonumber\\
    &\leq \left|\left|x_k-x^{\star}\right|\right|^2
    -2\eta^2\langle F(x_{k+1}), F(x_{k+1}) \rangle
    +\eta^2\left|\left|F(x_{k+1})\right|\right|^2
    \nonumber\\
    &= \left|\left|x_k-x^{\star}\right|\right|^2
    -\eta^2\left|\left|F(x_{k+1})\right|\right|^2.
    \label{eq:lec3-proximalpoint-descentlemma}
\end{align}
To arrive at (a), we added and subtracted 
$\langle F(x_{k+1}), x_{k+1}\rangle$ from the line before it.
The second term in (a) is nonpositive by monotonicity of $F$;
we dropped it, and this resulted in the inequality on the next line.
We used \eqref{eq:lec3-prox-point-update} 
to replace $x_k-x_{k+1}$ with $\eta F(x_{k+1})$ in the third term of (a).

We have proved the following descent lemma:
\begin{lemma}[Descent lemma for the proximal point method on monotone operators]
\label{lem:lec3-prox-point-descent-lemma}
Consider a monotone operator $F$. Under the dynamics of Algorithm \ref{alg:lec3-proximal-point} with step size $\eta$,
\begin{equation}
    \left|\left|x_{k+1}-x^{\star}\right|\right|^2
    \leq
    \left|\left|x_k-x^{\star}\right|\right|^2
    -\eta^2\left|\left|F(x_{k+1})\right|\right|^2.
\end{equation}
\end{lemma}

Lemma \ref{lem:lec3-prox-point-descent-lemma} guarantees that the distance
to the fixed point shrinks on every iteration of the proximal point
method, and therefore implies convergence.

In the interests of completeness, we will now
give some more vocabulary and state an important theorem concerning
monotone operators.

Let us give the inverse operator $\left(I+\eta F\right)^{-1}$ a name.
We call it the \emph{backward operator} $B$. $B$ is also known as the 
\emph{resolvent} of $F$. We can also define the operator $2B-I$.
See Fig.\,\ref{fig:firmly} for some intuition for the distinction
between the results of applying $B$ versus $2B-I$ to $x_{k}$.

\begin{figure}[!h]
    \centering
    
\begin{tikzpicture}[scale=1, transform shape]
\tikzstyle{point1}=[ball color=cyan, circle, draw=black, inner sep=0.1cm]
    \def\ra{1.86}; % radius of smaller circle
    \def\rb{2.5}; % radius of larger circle

    % Define center
    \coordinate (O) at (0,0);

    % Draw smaller circle
    \draw[name path=incircle,line width=0.5mm] (O) circle (\ra);

    % Draw larger circle
    \draw[line width=0.5mm] (O) circle (\rb);

    % Label center point
    % \node at (O) {$x^*$};
    \draw [fill] (O) circle (.05);
    \coordinate[label = right:$x^{\star}$] (O);

    % Define point on larger circle
    \coordinate (A) at (120:\rb);

    % Label point A
    \node[above left] at (A) {$x_k$};
    \draw [fill] (A) circle (.05);
    % \coordinate[label = left:$x_k$] (A);

    % Draw tangent at A
    \draw[line width=0.5mm,->] (A) -- ($(A)!3cm!-90:(O)$) node[pos=0.3,above left] {$-F(x_k)$} coordinate (L);
    
    \coordinate(B) at ($3/4*(L)+1/4*(A)$);
    
    % Label point B
    \node[above left] at (B) {$x^{GD}_{k+1}$};
    \draw [fill] (B) circle (.05);

    % Draw perpendicular line at B
    \draw[line width=0.5mm,->] (B) -- ($(B)!3cm!-90:(O)$) node[midway,left] {$-F(x^{GD}_{k+1})$} coordinate (C);

    % Draw dotted parallel line through A
    \coordinate
    (direction) at ($ (C) - (B) $);
    \draw[name path=dotline,line width=0.5mm,->,dotted] (A) -- ($(A) + (direction)$);
    \path[name intersections={of=incircle and dotline}];
    \node at (intersection-1) [right] {$x^{EG}_{k+1}$};
    \draw [fill] (intersection-1) circle (.05);
\end{tikzpicture}
    \caption{
    Following the flow of $-F$ at the point $x_k$ leads us away from $x^{\star}$, to a point $x_{k+1}^{GD}$ (the superscript GD stands for gradient descent). However, taking a step in the direction of the flow $-F$ at $x_{k+1}^{GD}$ \emph{from $x_k$} brings us to $x_{k+1}^{EG}$, which is closer to $x^{\star}$ than $x_k$. This is the general logic behind the proximal point method \eqref{eq:lec3-prox-point-update}, and is precisely what the extragradient algorithm does (write $\Tilde{x}_{k+1}=x_{k+1}^{GD}$ and $x_{k+1}=x_{k+1}^{EG}$ in Algorithm \ref{alg:lec3-extragradient} to match the notation in this figure).
    Hence the superscript in $x^{EG}_{k+1}$.
    }
    \label{fig:lec3-EG}
\end{figure}
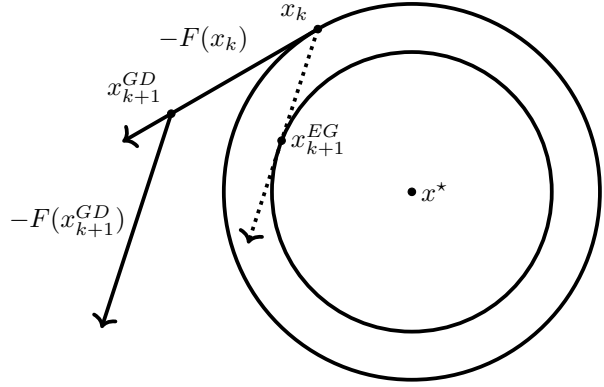

Now we can state a theorem that connects the properties of $F$, $B$ and $2B-I$.
\begin{theorem}
\label{thm:lec3-F-and-B}
    The following are equivalent.
    \begin{itemize}
        \item $F$ is a monotone operator.
        \item There exists a firmly non-expansive operator $B$ such that $B$ is the resolvent of $F$.
        \item The operator $2B-I$ is non-expansive.
    \end{itemize}
\end{theorem}
This theorem and its proof can be found in \cite{bauschke2011convex}
(see Proposition 4.4 and Corollary 23.9 therein).

The utility of Theorem \ref{thm:lec3-F-and-B} for our purposes is that given a monotone
operator $F$, we are guaranteed that the resolvent exists, and so we can always construct a proximal point method for $F$. 
Since the resolvent is also firmly non-expansive,
we are further guaranteed that the proximal point method will converge.
To see why, and for completeness, we end this section with the 
definition of firm non-expansivity.
\begin{definition}[Firm Non-expansivity]
An operator $B$ is \emph{firmly non-expansive} on a set $\mathcal{X}$ if
\begin{equation}
    \left|\left|B(x)-B(y)\right|\right|^2 + \left|\left|(I-B)(x)-(I-B)(y)\right|\right|^2 \leq \left|\left|x-y\right|\right|^2 \,\, \forall x,y\in\mathcal{X}.
    \label{eq:lec3-firmly-nonexpansive}
\end{equation}
\end{definition}
In words, firm non-expansivity guarantees that applying the operator to
two points strictly decreases the distance between. The amount of the decrease is 
given by the second term on the left-hand side of \eqref{eq:lec3-firmly-nonexpansive}.

\subsubsection{The Extragradient Algorithm}

Let us now address the implicit nature of the proximal point update
\eqref{eq:lec3-prox-point-update}
by giving a practical algorithm that approximates it---the extrapolated
gradient algorithm \cite{korpelevich1976extragradient}, or, as 
it is often abbreviated, the \emph{extragradient algorithm}.
The extragradient algorithm is a two-step method in which we compute
an intermediate iterate $\Tilde{x}_k$ by applying a gradient-descent-like 
update to $x_k$, and then apply another such update to arrive at 
$x_{k+1}$.
In the second update step, $F$ is evaluated at $\Tilde{x}_k$.
This is illustrated in Fig.\,\ref{fig:lec3-EG}.

\begin{algorithm}
\caption{Extragradient Method}
\label{alg:lec3-extragradient}
\begin{algorithmic}
\Require $\eta > 0$\\
\State $x_{k+1} = x_k - \eta F(\Tilde{x}_k),$ where\\
\\
$\Tilde{x}_k = x_k - \eta F(x_k)$
\end{algorithmic}
\end{algorithm}

The extragradient update can explicitly be shown to be an approximation
to the proximal point update \cite{mokhtari2020unified}, but we will not
do that here. Nonetheless, we will leverage the idea to prove convergence of 
the extragradient method by bounding the discrepancy between the iterates
it produces and the iterates that would be produced if we could exactly
solve the proximal point update.

We give the proof of convergence of the extragradient method in the 
case where $F$ is strongly monotone (so we can easily get a rate) and $L$-Lipschitz.
\begin{align}
    \left|\left|x_{k+1}-x^{\star}\right|\right|^2
    &= \left|\left|x_k-x^{\star}\right|\right|^2
    -2\eta\langle F(\Tilde{x}_k), x_k-x^{\star} \rangle
    +\eta^2 \left|\left| F(\Tilde{x}_k)\right|\right|^2
    \nonumber\\
    &\overset{(a)}{=} \left|\left|x_k-x^{\star}\right|\right|^2
    -2\eta\langle F(\Tilde{x}_k), \Tilde{x}_k - x^{\star} \rangle
    -2\eta\langle F(\Tilde{x}_k), x_k - \Tilde{x}_k \rangle
    +\eta^2 \left|\left| F(\Tilde{x}_k)\right|\right|^2
    \nonumber\\
    &\overset{(b)}{\leq} \left|\left|x_k-x^{\star}\right|\right|^2
    -2\eta\mu \left|\left|\Tilde{x}_k - x^{\star}\right|\right|^2
    +\left|\left|x_k-\Tilde{x}_k-\eta F(\Tilde{x}_k)\right|\right|^2
    -\left|\left|x_k-\Tilde{x}_k\right|\right|^2
    \nonumber\\
    &\overset{(c)}{\leq} \left|\left|x_k-x^{\star}\right|\right|^2
    -2\eta\mu \left|\left|\Tilde{x}_k - x^{\star}\right|\right|^2
    +\left(\eta^2 L^2 - 1\right)\left|\left|x_k-\Tilde{x}_k\right|\right|^2
    \nonumber\\
    &\overset{(d)}{\leq} \left(1-\eta\mu\right) \left|\left|x_k-x^{\star}\right|\right|^2
    + \left(\eta^2 L^2 - 1 + 2\eta\mu\right)\left|\left|x_k-\Tilde{x}_k\right|\right|^2
    \nonumber\\
    &\overset{(e)}{\leq} \left(1-\frac{\mu}{2(L+\mu)}\right) \left|\left|x_k-x^{\star}\right|\right|^2 \,\,\,\text{for}\,\, \eta=\frac{1}{2(L+\mu)}
    \nonumber\\
    &\overset{(f)}{\leq} \left(1-\frac{\mu}{4L}\right) \left|\left|x_k-x^{\star}\right|\right|^2.
\end{align}
In (a), we added and subtracted
$2\eta\langle F(\Tilde{x}_k), \Tilde{x}_k\rangle$, our usual trick. 
In (b), we applied the strong monotonicity property to the second term, and
completed the square with the last two terms of the previous line.
To arrive at the third term in (c), we manipulated the third term in (b) as follows: first, we
used the dynamics of the algorithm to replace $x_k-\Tilde{x}_k$ with $\eta F(x_k)$,
second, we applied the Lipschitz property to the quantity $||F(x_k)-F(\Tilde{x}_k)||^2$,
and third, we grouped the resulting term, $\eta^2 L^2 ||x_k-\Tilde{x}_k||^2$, with the last
term of line (b).
To arrive at (d), we applied the triangle inequality, 
$-2\eta\mu||a||^2 \leq -\eta\mu||a+b||^2 + 2\eta\mu||b||^2$,
with $a=\Tilde{x}_k - x_k + x_k - x^{\star}$ and $b=x_k-\Tilde{x}_k$.
In (e), we took $\eta=[2(L+\mu)]^{-1}$. This ensures that the
prefactor of the error term $||x_k-\Tilde{x}_k||^2$ in (d)
is negative, and that the term can therefore be dropped.
Note that with this assignment for the step size, line (d) is a descent
lemma for the extragradient method.
Lastly, to simplify the rate of convergence in (f), we used 
the fact that $2L\geq L+\mu$, which is easily proved as follows:
\begin{equation}
    \mu
    \leq\frac{\langle F(x)-F(y), x-y\rangle}{||x-y||^2}
    \leq\frac{|| F(x)-F(y)||}{||x-y||}
    \leq L.
    \label{eq:lec3-mu-is-less-than-L}
\end{equation}
In \eqref{eq:lec3-mu-is-less-than-L}, the first inequality
follows from strong monotonicity, the second from the Cauchy-Schwarz
inequality, and the third from the Lipschitz property of $F$.

We have proved the following theorem:
\begin{theorem}
    For a $\mu$-strongly monotone and $L$-Lipschitz operator $F$, 
    the iterates of Algorithm \ref{alg:lec3-extragradient} with step size 
    $\eta=1/[2(\mu+L)]$ converge to the fixed point of $F$ at the linear
    rate $1-\mu/(4L)$. In particular,
    \begin{equation}
        \left|\left|x_{k+1}-x^{\star}\right|\right|^2
        \leq
        \left(1-\frac{\mu}{4L}\right)^t \left|\left|x_1-x^{\star}\right|\right|^2.
    \end{equation}
\end{theorem}

The extragradient algorithm is one practical instantiation of the
proximal point method. There are others. We will mention some of them 
in the next section. We end this section by noting that
one way to construct approximate proximal point algorithms is by 
expanding the resolvent of $F$ in powers of $\eta$, truncating the 
expansion after a fixed number of terms $m$, and then rearranging 
the result to solve for $x_{k+1}$.
The accuracy of the approximation to the proximal point update can be
controlled by increasing $m$.

\subsubsection{High-Resolution Continuous-Time Limits}
\label{sec:lec3-highres-ODE}

We studied the phenomenon of acceleration in optimization from a 
continuous-time perspective in 
Section \ref{sec:lec2-cnts-time-acceleration}. 
We end these notes by illustrating with an example the 
value of thinking in continuous time about fixed-point problems.

We return to the two-player zero-sum game \eqref{eq:lec2-zero-sum-game}
to contextualize the discussion, which is based on work done in
\cite{chavdarova2022lastiterate}. We remind the reader that 
the vector field $F$ of the game \eqref{eq:lec2-zero-sum-game}
is given in \eqref{eq:lec2-two-player-game-vector-field}.
There are several algorithms that can be used to compute the fixed 
points of $F$. 
We consider four of them: (1) gradient-descent-ascent (GDA), 
(2) the extragradient algorithm (EG), (3) optimistic GDA (OGDA), and 
(4) the lookahead algorithm (LA) \cite{chavdarova2020taming}.
Writing the state vector of the two players $x_1$ and $x_2$ together as 
$z=\begin{pmatrix}x_1\\x_2\end{pmatrix}$,
the update rules for all four algorithms are given below:
\begin{align}
    \text{GDA}: \,\,\,& z_{k+1} = z_k - \gamma F(z_k),
    \label{eq:lec3-GDA}\\
    \text{EG}: \,\,\,& z_{k+1} = z_k - \gamma F(\Tilde{z}_k),
    \label{eq:lec3-EG}\\
    \text{OGDA}: \,\,\,& z_{k+1} = z_k - 2\gamma F(z_k) + \gamma F(z_{k-1}),\label{eq:lec3-OGDA}\\
    \text{LA}: \,\,\,& z_{k+1}  = z_k + \alpha\left(\Tilde{z}_{k+\ell}-z_{k}\right), \, \, \alpha\in (0,1].\label{eq:lec3-LA}
\end{align}
In \eqref{eq:lec3-EG}, $\Tilde{z}_k$ is computed as in Algorithm \ref{alg:lec3-extragradient}.
In \eqref{eq:lec3-LA}, the $\Tilde{z}_{k+\ell}$ is computed
by applying $\ell$ iterations of some base algorithm. 
Here, we will take the base algorithm to be GDA.

Now we might ask the following questions---what are the relative merits
of each of these algorithms and how might we pick the most appropriate
among them in the context of a specific problem?
We have seen previously how taking a continuous-time perspective
in optimization can be productive of insight into such questions, and
so we take that perspective again here.
An interesting technical roadblock arises, however, which is that
when we naively take the continuous-time
limits of \eqref{eq:lec3-GDA}-\eqref{eq:lec3-LA}, 
they all collapse to a single differential equation.
This is particularly problematic because it is known that GDA can
diverge on two-player zero-sum games \cite{daskalakis2018training,chavdarova2022lastiterate},
but that EG, OGDA, and LA do not.
The continuous-time limit of these algorithms therefore fails to 
distinguish between them even though they display distinct convergence 
behaviors in discrete time.

The solution to this problem is to take more careful continuous-time
limits. We take what is called a \emph{high-resolution limit}, which 
has its roots in the study of hydrodynamics, and is useful to study
systems that contain multiple timescales. Here is the
basic idea: Typically, we are faced with an algorithm of the form
$x_{n+1} = G(x_n, \eta)$ where $G$ is the update rule and $\eta$
is the step size. In order to arrive at a continuous-time representation
$x(t)$, we make the association $t=n\eta$ and take the limit 
$\eta\rightarrow 0$. The idea behind a high-resolution limit is to 
instead take $t=n g(\eta)$ where $g$ is a function (other than the 
identity function) of $\eta$.

The high-resolution continuous-time limits of GDA, EG, OGDA, and LA
\emph{are} distinct. Writing $\beta=2/\eta$, they are
\begin{align}
    \text{GDA}: \,\,\,& \ddot{z}(t) = -\beta\dot{z}(t) - \beta F(z(t)),\label{eq:lec3-GDA-highres}\\
    \text{EG}: \,\,\,& \ddot{z}(t) = -\beta\dot{z}(t) - \beta F(z(t)) + 2 J(z(t))F(z(t)),\label{eq:lec3-EG-highres}\\
    \text{OGDA}: \,\,\,& \ddot{z}(t) = -\beta\dot{z}(t) - \beta F(z(t)) - 2 J(z(t))\dot{z}(t),\label{eq:lec3-OGDA-highres}\\
    \text{LA,} \,\ell=2: \,\,\,& \ddot{z}(t) = -\beta\dot{z}(t) - 2\alpha\beta F(z(t)) + 2 \alpha J(z(t))F(z(t)).\label{eq:lec3-LA-highres}
\end{align}
In \eqref{eq:lec3-EG-highres}-\eqref{eq:lec3-LA-highres}, we note
the appearance of the Jacobian $J$ of the vector field,
\begin{equation}
    J(z)
    =
    \begin{bmatrix}
    \nabla^2_{x_1}f(z) & \nabla_{x_2}\nabla_{x_1}f(z)\\
    -\nabla_{x_1}\nabla_{x_2}f(z) & -\nabla^2_{x_2}f(z)
    \end{bmatrix},
    \label{eq:lec3-two-player-game-J}
\end{equation}
which is not present in the high-resolution continuous-time limit
of GDA. This suggests that the three convergent algorithms all leverage
information contained in the Jacobian of $F$ in order to converge.
It is also interesting that the continuous-time limit of LA is $\ell$-dependent. 
The structure of the continuous-time equations \eqref{eq:lec3-GDA-highres}-\eqref{eq:lec3-LA-highres} thus immediately give clues to understanding the
different convergence properties of the various algorithms and motivates
new questions for future research.

\bibliographystyle{unsrt}
\bibliography{bibliography.bib}

\end{document}